\documentclass{tlp}

\usepackage{aopmath}

\usepackage{url}
\usepackage{epsfig}
\usepackage{graphics}
\usepackage{graphicx}

\newtheorem{example}{Example} 

\newcommand{\sem}{\mbox{\sc sem}}
\newcommand{\stb}{\mbox{\sc stable}}
\newcommand{\sppm}{\mbox{\sc suppmin}}
\newcommand{\spp}{\mbox{\sc supp}}
\newcommand{\caHB}{\mathcal{HB}}

\newcommand{\caC}{\mathcal{C}}

\newcommand{\ac}{{A^c}}
\newcommand{\bc}{{B^c}}

\newcommand{\Yy}{{Y\setminus\{y\}}}
\newcommand{\Xy}{{X\setminus\{y\}}}

\newcommand{\At}{\mathit{At}}
\newcommand{\at}{\mathit{At}}
\newcommand{\Mod}{\mathit{Mod}}
\newcommand{\SE}{\mathit{SE}}

\newcommand{\hd}{\mathit{H}}
\newcommand{\sh}{\mathit{sh}}

\newcommand{\bd}{\mathit{B}}
\newcommand{\nt}{\mathit{not\,}}
\newcommand{\vph}{\varphi}

\newcommand{\lar}{\leftarrow}
\newcommand{\lpor}{\vee}
\newcommand{\st}{\;|\,}
\newcommand{\Pol}{{\mathit{Pol}}}
\newcommand{\PiP}[1]{\Pi^P_{#1}}

\begin{document}
\label{firstpage}

\bibliographystyle{acmtrans}

\long\def\comment#1{}

\title[Hyperequivalence for Modular Logic Programming]
{Relativized Hyperequivalence of Logic Programs for Modular Programming}

\author[M. Truszczy\'nski and S. Woltran]
{MIROS\L AW TRUSZCZY\'NSKI\\
Department of Computer Science, University of Kentucky, Lexington, KY 40506-0046, USA\\
\email{mirek@cs.uky.edu}
\and
STEFAN WOLTRAN\\
Institute for Information Systems 184/2, Technische Universit{\"a}t Wien,\\
Favoritenstrasse 9-11, 1040 Vienna, Austria\\
\email{woltran@dbai.tuwien.ac.at}
}

\pagerange{\pageref{firstpage}--\pageref{lastpage}}
\volume{\textbf{YY} (Z):}
\jdate{Month 20XX}
\setcounter{page}{1}
\pubyear{20XX}

\submitted {21 January 2009}
 \revised {24 June 2009}
 \accepted{21 July 2009}

\maketitle

\begin{abstract}
A recent framework of relativized hyperequivalence of programs offers a
unifying generalization of strong and uniform equivalence. It seems to
be especially well suited for applications in program optimization and
modular programming due to its flexibility that allows us to restrict,
independently of each other, the head and body alphabets in context
programs. We study relativized hyperequivalence for the three
semantics of logic programs given by stable, supported
and supported minimal models. For each semantics, we identify four types
of contexts, depending on whether the head and body alphabets are given
directly or as the \emph{complement} of a given set. Hyperequivalence
relative to contexts where the head and body alphabets are specified
directly has been studied before. In this paper, we establish the
complexity of deciding relativized hyperequivalence with respect to
the three other types of context programs.
\end{abstract}


\section{Introduction}

We study variants of relativized hyperequivalence that
are relevant for the development and analysis of disjunctive logic programs
with modular structure. Our main results concern the complexity of 
deciding relativized hyperequivalence for the three major semantics of 
logic programs given by stable, supported and supported minimal 
models.

\emph{Logic programming} with the semantics of stable models, nowadays often 
referred to as \emph{answer-set programming}, is a computational 
paradigm for knowledge representation, as well as modeling and 
solving constraint problems \cite{mt99,nie99,GelLeo02,baral03}. In
recent years, it has been steadily attracting more attention. One 
reason is that answer-set programming is truly declarative. Unlike
in, say, Prolog, the order of rules in programs and the order of
literals in rules have no effect on the meaning of the program. 
Secondly, the efficiency
of the latest tools for processing programs, especially solvers, reached
the level that makes it feasible to use them for problems of practical
importance \cite{
asp-c07}. 

It is broadly recognized in software engineering that modular programs
are easier to design, analyze and implement. Hence, essentially all 
programming languages and environments support the development of 
modular programs. Accordingly, there has been much work recently to 
establish foundations of \emph{modular} answer-set programming. One 
line of investigations has focused on the notion of an answer-set 
program \emph{module} 
\cite{Gelfond02-CL,Janhunen06,Oikarinen06,jotw07}. This work 
builds on ideas for 
compositional semantics of logic programs proposed by 
\citeNS{GaifmanS89} and
encompasses earlier results on stratification 
and 
\emph{program splitting} \cite{litu94}. 

The other main line of research, to which our paper belongs, has
centered on program equivalence and, especially, on the concept of
equivalence for substitution. 
Programs $P$ and 
$Q$ are \emph{equivalent for substitution} with respect to a class 
$\caC$ of programs called \emph{contexts}, if for every context $R\in
\caC$, $P\cup R$ and $Q\cup R$ have the same stable models. Thus, if a
logic program is the union of programs $P$ and $R$, where $R\in \caC$,
then $P$
can be replaced with $Q$, with the guarantee that the semantics is
preserved no matter what $R$ is (as long as it is in $\caC$)
precisely when $P$ and $Q$ are equivalent for substitution with respect to
$\caC$. If $\caC$ contains the empty program (which is typically the case
{and, in particular, is the case for the families of programs we 
consider in the paper}),
the equivalence for substitution with respect to $\caC$ implies the standard 
equivalence under the 
stable-model semantics.\footnote{Two programs are equivalent under the 
stable-model semantics if they have the same stable models.} \emph{The
converse is not true}. We refer to these 
stronger
forms of equivalence collectively as \emph{hyperequivalence}. 

Hyperequivalence with respect to the class of \emph{all}
programs, known more commonly as \emph{strong equivalence}, was
proposed and studied by 
\citeNS{lpv01}. That work prompted extensive 
investigations of the concept that resulted in new characterizations 
\cite{lin02,tu03} and connections to certain non-standard logics 
\cite{dejo-hend-03}. 
Hyperequivalence with respect to contexts consisting of facts was
studied by \citeNS{ef03}. This version of 
hyperequivalence, known
as \emph{uniform equivalence}, appeared first in the database area in
the setting of DATALOG and query equivalence~\cite{sagi-88}. 
Hyperequivalence with respect to contexts restricted to a given 
alphabet, or \emph{relativized} hyperequivalence, was
proposed 
by \citeNS{Woltran04} and \citeNS{inou-saka-04}.
Both uniform equivalence and relativized hyperequivalence were analyzed 
in depth
by \citeNS{efw04}, and 
later generalized by Woltran~\citeyear{Woltran07} to allow contexts that 
use (possibly) different
alphabets for the heads and bodies of rules. That approach offers 
a unifying
framework for strong and uniform equivalence.
Hyperequivalence, in which one compares projections of answer sets
on some designated sets of atoms rather than entire answer sets
has also received some attention \cite{EiterTW05,Oetsch07}.

All those results concern the stable-model semantics of programs. There
has been little work on other semantics, with the work by 
\citeNS{cab06} long being a
notable single exception. Recently however, 
\citeNS{tw08a} introduced and 
investigated relativized hyperequivalence of programs under the
semantics of supported models \cite{cl78} and supported minimal models,
two other major semantics of logic programs. 
\citeNS{tw08a} characterized these variants of hyperequivalence and established
the complexity of some associated decision problems.

In this paper, we continue research of relativized hyperequivalence 
under all three major semantics of logic programs. As in earlier works
\cite{Woltran07,tw08a}, we focus on contexts
of the form $\caHB(A,B)$, where $\caHB(A,B)$ stands for the set of all
programs that use atoms from $A$ in the heads and 
atoms from $B$ in the bodies of rules. Our main goal is to establish the
complexity of deciding whether two programs are hyperequivalent 
(relative to a specified semantics) with respect to $\caHB(A,B)$. We
consider the cases when $A$ and $B$ are either specified directly or in
terms of their complement. As we point out in the following section, 
such contexts arise naturally when we design modular logic programs.

\section{Motivation}

We postpone technical preliminaries to the following section. For the
sake of the present section it is enough to say that we focus our study
on finite propositional programs over a fixed countable 
{infinite}
set $\at$ of atoms.
It is also necessary to introduce one piece of notation: $X^c=\at\setminus X$.

To argue that contexts specified in terms of the complement of a finite
set are of interest, let us consider the following scenario. A logic
program is \emph{$A$-defining} if it specifies the definitions
of atoms in $A$. The definitions may be recursive, they may involve 
\emph{interface} atoms, that is, atoms defined in other modules 
{(such 
atoms facilitate importing information form other modules, hence the term
``interface'')}, as well as atoms used locally to represent some needed
auxiliary concepts. 
{Let $P$ be a particular $A$-defining program with 
$L$ as the set of its local atoms.}
For $P$ to behave properly when combined with other programs, these 
``context'' programs must not have any 
occurrences of atoms from $L$ and must have no atoms from $A$ in the 
heads of their rules. In our terminology, these are precisely programs 
in $\caHB((A\cup L)^c,L^c)$.\footnote{$A$-defining programs were introduced
by 
\citeNS{ErdoganL03a}. However, that work considered more restricted
classes of programs with which $A$-defining programs could be combined.}

The definitions of atoms in $A$ can in general be captured by several
different $A$-defining programs. A key question concerning
such programs is whether they are equivalent. Clearly, two $A$-defining
programs $P$ and $Q$, both using atoms from $L$ to represent local 
auxiliary concepts, should be regarded as equivalent if they behave in 
the same way in the context of any program from $\caHB((A\cup L)^c,L^c)$.
In other words, the notion of equivalence appropriate in our
setting is that of hyperequivalence with respect to $\caHB((A\cup L)^c,L^c)$ under a selected 
semantics (stable, supported or supported-minimal).

\begin{example}
Let us assume that $A=\{a,b\}$ and that $c$ and $d$ are interface atoms
(atoms defined elsewhere). We need a module that works as follows:
\begin{enumerate}
\item If $c$ and $d$ are both true, exactly one of $a$ and $b$ must be
true
\item If $c$ is true and $d$ is false, only $a$ must be true
\item If $d$ is true and $c$ is false, only $b$ must be true
\item If $c$ and $d$ are both false, $a$ and $b$ must be false.
\end{enumerate}
We point out that $c$ and $d$ may depend on $a$ and $b$ and so, 
in some cases the overall program may have no models of a particular 
type (to be concrete, for a time being we fix attention to stable 
models). 

One way to express 
the conditions (1) - (4)
is by means of the 
following $\{a,b\}$-defining program $P$
(in this example we assume that $\{a,b\}$-defining programs do not
use local atoms, that is, $L=\emptyset$):

\smallskip
\noindent
\hspace*{0.2in}$a \lar c, \nt b;$\\
\hspace*{0.2in}$b \lar d, \nt a$.

\smallskip
\noindent
Combining $P$ with programs that 
specify facts: $\{c,d\}$, $\{c\}$,
$\{d\}$ and $\emptyset$, it is easy to see that $P$ behaves as required. 
For instance, $P\cup\{c\}$ has exactly one stable model $\{a,c\}$. 

However, $P$ may also be combined with more complex programs. For 
instance, let us consider the 
program $R=\{c \lar \nt d;\; d \lar a, \nt c\}$.
Here, $d$ can only be true if $a$ is true and $c$ is false, which is
impossible given the way $a$ is defined. 
Thus, $d$ must be false and $c$ must be true. 
According to the specifications, there should be exactly one stable model for
$P\cup R$ in this case: $\{a,c,d\}$. It is easy to verify that it is 
indeed the case. 

The specifications for $a$ and $b$ can also be expressed by other 
$\{a,b\}$-defining programs, in particular, by the following program $Q$:

\smallskip
\noindent
\hspace*{0.2in}$a \lar c,d, \nt b;$\\
\hspace*{0.2in}$b \lar c,d, \nt a;$\\
\hspace*{0.2in}$a \lar c, \nt d;$\\
\hspace*{0.2in}$b \lar d, \nt c$.

\smallskip
\noindent
The question arises whether $Q$ behaves in the same way as $P$ relative 
to programs from $\caHB(\{a,b\}^c,\emptyset^c)=\caHB(\{a,b\}^c,\at)$.
For all contexts considered earlier, it is the case. However, in general, 
it is not so. For instance, 
if $R=\{c\lar\;;\;  d \lar a\}$
then, $\{a,c,d\}$ is a stable model of $P\cup R$, while $Q\cup R$ has no 
stable models. Thus, $P$ and $Q$ cannot be viewed as equivalent 
$\{a,b\}$-defining programs. \hspace*{1em}\proofbox
\end{example}

A similar scenario
gives rise to a different class of contexts.
We call a 
program $P$ \emph{$A$-completing} if it completes partial and non-recursive 
definitions of atoms in $A$ given by other modules 
{which, for 
instance, might
specify the base conditions for a recursive definition of atoms in $A$. 
Any program with all atoms in the heads of rules in $A$ can be regarded
as an $A$-completing program. Assuming that $P$ is an $A$-completing 
program (again with $L$ as a set of local atoms), $P$ can be combined with 
any program $R$ that has no occurrences of atoms from $L$ and no occurrences 
of atoms from $A$ in the bodies of its rules. However, atoms from $A$ may 
occur in the heads of rules from $R$, which constitute a partial, 
non-recursive part of the definition of $A$, ``completed'' by $P$. Such 
programs $R$ form precisely the class $\caHB(L^c,(A\cup L)^c)$.} 

Finally, let us consider a situation where we are to express 
partial problem specifications as a logic program. 
In that program, we need to use
concepts represented by atoms from some set $A$ that
are defined elsewhere in terms of concepts described by atoms
from some set $B$. Here two programs $P$ and $Q$ expressing these partial 
specifications can serve as each other's substitute precisely when they 
are hyperequivalent with respect to the class of programs $\caHB(A,B)$.
 
These examples demonstrate that
hyperequivalence with respect to context classes $\caHB(A,B)$, where
$A$ and $B$ are either specified directly or in terms of their
complement is of interest.
Our goal is to study the complexity of deciding whether two
programs are hyperequivalent relative to such classes of contexts.

\section{Technical Preliminaries}

\textbf{Basic logic programming  notation and definitions.}
We recall that we consider a fixed countable infinite set of propositional
atoms $\At$. 
\emph{Disjunctive logic programs} (programs, for short) are finite sets of (program) \emph{rules} ---
expressions of the form
\begin{equation}
\label{eqA}
a_1 \lpor\ldots\lpor a_k\lar b_1,\ldots,b_m,\nt c_1 ,\ldots,\nt c_n,
\end{equation}
where $a_i$, $b_i$ and $c_i$ are atoms in $\At$, `$\lpor$' stands for the
disjunction, `,' stands for the conjunction, and $\nt$ is the 
\emph{default} negation. If $k=0$, the rule is a \emph{constraint}. If
$k\leq 1$, the rule is \emph{normal}. 
Programs consisting of normal rules 
are called \emph{normal}. 

We often write the rule (\ref{eqA}) as $H\lar B^{+},\nt B^{-}$, where $H=
\{a_1,\ldots,a_k\}$, $B^{+}=\{b_1,\ldots,b_m\}$ and $B^{-}=\{c_1,\ldots,
c_n\}$. We call $H$ the \emph{head} of the rule, and the conjunction
$B^{+},\nt B^{-}$, the \emph{body} of the rule. The sets $B^{+}$ and $B^{-}$
form the positive and negative body of the rule. Given a rule $r$, we 
write $H(r)$, $B(r)$, $B^{+}(r)$ and $B^{-}(r)$ to denote the head, the 
body, the positive body and the negative body of $r$, respectively.
For a program $P$, we
set $\hd(P)=\bigcup_{r\in P}H(r)$, $B^\pm(P)= \bigcup_{r\in P}
(B^{+}(r)\cup B^{-}(r))$, and $\At(P)=\hd(P)\cup B^\pm(P)$.

For an interpretation $M\subseteq \At$ and a rule $r$, we define 
entailments $M\models B(r)$, $M\models H(r)$ and $M\models r$ in the 
standard way. That is, $M\models B(r)$, if jointly 
$B^{+}(r)\subseteq M$ and $B^{-}(r)\cap M=\emptyset$; 
$M\models H(r)$, if $H(r)\cap M\neq \emptyset$;
and $M\models r$, if $M\models B(r)$  implies $M\models H(r)$.
An interpretation $M\subseteq \At$ is a \emph{model}
of a program $P$ ($M\models P$), if $M\models r$ for every $r\in P$.

The \emph{reduct} of a disjunctive logic program $P$ with respect to
a set $M$ of atoms, denoted by $P^M$, is the program 
$\{ \hd(r)\lar \bd^{+}(r) \mid r\in P,\ M\cap B^{-}(r)=\emptyset\}$.
A set $M$ of atoms is a \emph{stable model}
of $P$ if $M$ is a minimal model (with respect to inclusion) of $P^M$.

If a set $M$ of atoms is a minimal hitting set of $\{\hd(r)\st r\in 
P, \ 
M\models\bd(r)\}$, then $M$ is 
{called} a \emph{supported
model} of $P$ \cite{bradix96jlp1,Inoue98-JLP}.\footnote{{A set $X$ is
a \emph{hitting} set for a family $\mathcal{F}$ of sets if for every
$F\in\mathcal{F}$, $X\cap F\not=\emptyset$.}} In addition,
$M$ is 
{called} a \emph{supported minimal model} of $P$ if it is 
a supported 
model of $P$ and a minimal model of $P$.
{One can check that supported models of $P$ are indeed models of $P$.}

A stable model of a program is a supported model of the program and a 
minimal model of the program. Thus, a stable model of a program is 
a supported minimal model of the program. However, the converse does not 
hold in general.
Supported models of a \emph{normal} logic program $P$ have a useful
characterization in terms of the (partial) one-step provability
operator $T_P$, 
defined as follows. For $M \subseteq \At$,
if there is a constraint $r\in P$ such that $M\models \bd(r)$ (that is, 
$M\not\models r$), then $T_P(M)$ is undefined. Otherwise,
\(
T_P(M) = \{\hd(r)\st r\in P,\ 
M\models\bd(r)\}.
\)
Whenever we use $T_P(M)$ in a relation such as (proper) inclusion,
equality or inequality, we always implicitly assume that $T_P(M)$ is
defined.

It is well known that $M$ is a model of $P$ if and only if $T_P(M)
\subseteq M$ (which, according to our convention, is an abbreviation
for: $T_P$ is defined for $M$ and $T_P(M) \subseteq M$). 
Similarly, $M$ is a \emph{supported} model of $P$ if $T_P(M)=M$ \cite{ap90} 
(that is, if $T_P$ is defined for $M$ and $T_P(M)= M$).

For a rule $r=a_1\lpor\ldots\lpor a_k\lar \bd$,
where $k\geq 2$, a \emph{shift} of $r$ is a normal program
rule of the form
\[
a_i \lar \bd, \nt a_1,\ldots,\nt a_{i-1},\nt a_{i+1},\ldots,\nt a_k,
\]
where $i=1,\ldots,k$. If $r$ is normal, the only \emph{shift} of
$r$ is $r$ itself. A program consisting of all shifts of rules in a
program $P$ is the \emph{shift} of $P$. We denote it by $\sh(P)$. It is
evident that a set $M$ of atoms is a (minimal) model of $P$ if and only
if $M$ is a (minimal) model of $\sh(P)$. It is easy to check that $M$
is a supported (minimal) model of $P$ if and only if it is a supported (minimal) model of
$\sh(P)$.
Moreover, $M$ is a supported model of $P$ if and only if
$T_{\sh(P)}(M)=M$.

\smallskip
\noindent
\textbf{Characterizations of hyperequivalence of programs.}
Let $\caC$ be a class of (disjunctive) logic programs. 
Programs $P$ and
$Q$ are 
\emph{supp-equivalent} (\emph{suppmin-equivalent},
\emph{stable-equivalent}, 
respectively) relative to $\caC$ if for every
program $R\in\caC$, $P\cup R$ and $Q\cup R$ have the same 
supported  (supported minimal, stable, respectively) models. 

In this paper, we are interested in equivalence of all three types 
relative to classes of programs defined by the \emph{head} and 
\emph{body alphabets}. Let $A,B\subseteq\At$. By $\caHB(A,B)$ we 
denote the class of all programs $P$ such that $\hd(P)
\subseteq A$ and $\bd^\pm(P) \subseteq B$. 
Clearly, $\emptyset\in\caHB(A,B)$ holds, for arbitrary $A,B\subseteq \At$.
{Thus, as we noted in the introduction, for each of the semantics 
and every sets $A$ and $B$, the corresponding hyperequivalence implies 
the standard equivalence with respect to that semantics.}

When studying supp- and suppmin-equivalence we will restrict ourselves
to the case of normal programs. Indeed, disjunctive programs $P$ and 
$Q$ are supp-equivalent (suppmin-equivalent, respectively) with
respect to $\caHB(A,B)$ if and only if normal programs $\sh(P)$ and
$\sh(Q)$ are supp-equivalent (suppmin-equivalent, respectively) with
respect to $\caHB(A,B)$ \cite{tw08a}. Thus, from now on whenever we 
consider supp- and suppmin-equivalence, we implicitly assume that 
programs under comparison are normal. In particular, we use that convention
in the definition below and the subsequent theorem.


For supp-equivalence and suppmin-equivalence, we need the set $\Mod_A(P)$,
defined by 
\citeNS{tw08a}.
Given a program $P$, and a set $A\subseteq\At$, 
\[
\Mod_A(P)=\{Y\subseteq\At\st Y\models P\ \mbox{and} \ Y\setminus T_P(Y)
\subseteq A\}.
\]
{\citeNS{tw08a} explain that elements of $\Mod_A(P)$ can be viewed as
\emph{candidates} for becoming supported models of an extension
of $P$ by some program $R\in \caHB(A,B)$. Indeed, each such
candidate interpretation $Y$ has to be a classical model of $P$,
as otherwise it cannot be a supported model, no matter how $P$ is
extended. Moreover, the elements from $Y\setminus T_P(Y)$ have to be
contained in $A$, as otherwise programs from $\caHB(A,B)$ cannot close
this gap.} The set $\Mod_A(P)$ is the key to the characterization of 
supp-equivalence. 

\begin{theorem}
\label{nlp1}
Let $P$ and $Q$ be programs, $A\subseteq
\At$, and $\cal C$ a class of programs such that 
$\caHB(A,\emptyset) \subseteq\caC\subseteq \caHB(A,\at)$. Then,
$P$ and $Q$ are supp-equivalent relative to $\caC$ 
if and only if
$\Mod_A(P)=\Mod_A(Q)$ and for every $Y\in\Mod_A(P)$, $T_P(Y)=T_Q(Y)$.
\end{theorem}

To characterize suppmin-equivalence, 
we use the set $\Mod_A^B(P)$ \cite{tw08a}, which consists of all pairs
$(X,Y)$ such that 
\begin{enumerate}
\item $Y\in\Mod_A(P)$
\item $X\subseteq Y|_{A\cup B}$
\item for each $Z \subset Y$
such that
$Z|_{A\cup B} = Y|_{A\cup B}$, $Z\not\models P$
\item for each $Z \subset Y$ such that $Z|_B = X|_B$
and $Z|_A  \supseteq X|_A$, $Z\not\models P$
\item if $X|_B=Y|_B$, then $Y\setminus T_P(Y)\subseteq X$.
\end{enumerate}

\begin{theorem}\label{thm:general}
Let $A,B\subseteq \At$ and let $P,Q$ be programs. 
Then, $P$ and $Q$ are suppmin-equivalent relative to $\caHB(A,B)$
if and only if
$\Mod^B_A(P)=\Mod^B_A(Q)$ and for every $(X,Y)\in\Mod^B_A(P)$,
$T_P(Y)|_{B}=T_Q(Y)|_{B}$.
\end{theorem}

Relativized stable-equivalence of programs was characterized by 
\citeNS{Woltran07}. We define $\SE_A^B(P)$ to consist of all pairs
$(X,Y)$, where $X,Y\subseteq\at$, such that:\footnote{We use a slightly different presentation
than the one given by \citeNS{Woltran07}. It is equivalent to the original one.} 
\begin{enumerate}
\item $Y\models P$
\item $X=Y$, or jointly $X\subseteq Y|_{A\cup B}$ and $X|_A\subset Y|_A$
\item for each $Z \subset Y$
such that
$Z|_{A} = Y|_{A}$, $Z\not\models P^Y$
\item for each $Z \subset Y$ such that 
$Z|_B \subseteq X|_B$ and $Z|_A  \supset X|_A$, or 
$Z|_B \subset X|_B$ and $Z|_A \supseteq X|_A$, 
$Z\not\models P^Y$
\item there is $Z\subseteq Y$ such that $X|_{A\cup B}=Z|_{A\cup B}$ and $Z
\models P^Y$.
\end{enumerate}

\begin{theorem}\label{thm:stable}
Let $A,B\subseteq \At$ and let $P,Q$ be programs. 
Then, $P$ and $Q$ are stable-equivalent relative to $\caHB(A,B)$ 
if and only if $\SE^B_A(P)
=\SE^B_A(Q)$.
\end{theorem}

\smallskip
\noindent
\textbf{Decision problems.}
We are interested in problems of deciding hyperequivalence relative 
to classes of programs 
of the
form $\caHB(A',B')$, where $A'$ and $B'$ stand either for finite sets
or for complements of finite sets. In the former case, the set is given
\emph{directly}. In the latter, it is specified by means of its finite
\emph{complement}. Thus, we obtain the classes of \emph{direct-direct}, 
\emph{direct-complement}, \emph{complement-direct} and 
\emph{complement-complement} decision problems. We denote them using 
strings of the 
form $\sem_{\delta,\varepsilon}(\alpha,\beta)$, where
\begin{enumerate}
\item 
$\sem$ stands for $\spp$, $\sppm$ or $\stb$ and identifies the 
semantics relative to which we define hyperequivalence;
\item
$\delta$ and $\varepsilon$ stand for $d$ or $c$ (direct and
complement, respectively), and specify one of the four classes of 
problems mentioned above;
\item 
$\alpha$ is either $\cdot$ or $A$, where $A\subseteq\at$ is finite.
If $\alpha=A$, then $\alpha$ specifies a \emph{fixed} alphabet for the 
heads of rules in context programs: either $A$ or the complement $\ac$ 
of $A$, depending on whether $\delta=d$ or $c$. The 
parameter $A$ does not belong to and does not vary with input. If 
$\alpha = \cdot$, then the specification $A$ of the head alphabet is part of 
the input and defines it as $A$ or $\ac$, again according to
$\delta$;
\item
$\beta$ is either $\cdot$ or $B$, where $B\subseteq\at$ is 
finite. It obeys the same conventions as $\alpha$ but defines 
the body alphabet according to the value of $\varepsilon$.
\end{enumerate}

For instance, $\sppm_{d,c}(A,\cdot)$, where $A\subseteq\at$ is finite, 
stands for the following problem: given programs $P$ and $Q$, and a set 
$B$, decide whether $P$ and $Q$ are suppmin-equivalent with respect to
$\caHB(A,B^c)$. 
Similarly, $\stb_{c,c}(\cdot,\cdot)$ denotes the 
following problem: given programs $P$ and $Q$, and sets $A$ and 
$B$, decide whether $P$ and $Q$ are stable-equivalent with respect to 
$\caHB(A^c, B^c)$.
With some abuse of notation, we often talk about ``the problem 
$\sem_{\delta,\varepsilon}(A,B)$'' as a shorthand for ``an arbitrary problem 
of the form $\sem_{\delta,\varepsilon}(A,B)$ with fixed finite sets $A$ and $B$'';
likewise we do so for $\sem_{\delta,\varepsilon}(\cdot,B)$ and $\sem_{\delta,\varepsilon}(A,\cdot)$.

As we noted, for supp- and suppmin-equivalence, there is no essential
difference
between normal and disjunctive programs. For stable-equivalence, 
allowing disjunctions in the heads of rules affects the complexity. 
Thus, in the case of stable-equivalence, we distinguish versions
of the problems $\stb_{\delta,\varepsilon}(\alpha,\beta)$, where 
the input programs 
are normal.\footnote{%
As demonstrated by \citeNS{Woltran07}, we can also restrict the 
programs used as contexts to normal ones, 
as that makes no difference.} 
We denote these problems by $\stb^n_{\delta,\varepsilon}(\alpha,\beta)$.

Direct-direct problems for the semantics of supported and supported minimal
models were considered earlier \cite{tw08a}, and their complexity was fully 
determined there. The complexity of problems $\stb_{d,d}(\cdot,\cdot)$, was 
also established before \cite{Woltran07}. Problems similar to 
$\stb_{c,c}(A,A)$ were already studied by 
\citeNS{efw04}.
In this paper, we complete the results on the complexity of problems
$\sem_{\delta,\varepsilon}(\alpha,\beta)$ for all three semantics.
In particular, we establish the complexity of the problems with at least one of 
$\delta$ and $\varepsilon$ being equal to $c$. 
 
The complexity of problems involving the complement of 
$A$ or $B$ is not a straightforward consequence of the results on 
direct-direct problems. 
In the
direct-direct problems, the class of context programs is 
essentially finite, as the head and body alphabets
for rules are finite. It is no longer the case 
for the three remaining problems, where at least one of the alphabets is 
infinite and so, the class of contexts is infinite, as well.

We note that when we change $A$ or $B$ to $\cdot$ in the problem 
specification, the resulting problem is at least as hard as the original 
one. Indeed for each such pair of problems, there are straightforward
{polynomial-time} reductions from one to the other. We illustrate
 these relationships
in Figure~\ref{fig1}. 
{Each arrow indicates that the ``arrowtail'' problem can be reduced
in polynomial time to the ``arrowhead'' one.}
Consequently, if there is a path from a problem $\Pi$ to the
problem $\Pi'$ in the diagram, $\Pi'$ is at least as hard as $\Pi$ and
$\Pi$ is at most as hard as $\Pi'$. We use this observation in proofs of
all complexity results.

\begin{figure}[ht]
\centerline{\includegraphics[scale=0.4]{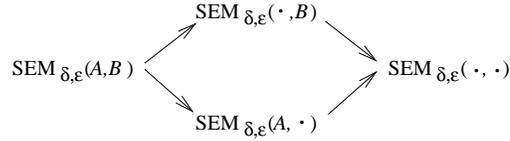}}
\caption{A simple comparison of the hardness of problems}
\label{fig1}
\end{figure}

{Finally, we note that throughout the paper, we write $\Pol$ 
instead of the more common P to denote the class of all problems that
can be solved by deterministic polynomial-time algorithms. As decision
problems we consider typically refer to a program $P$, we want to avoid
the ambiguity of using the same symbol in two different meanings.}

\section{Supp-equivalence}

As the alphabet for the bodies of context programs
plays no role in supp-equi\-valence (cf.\ Theorem~\ref{nlp1}), 
the problems 
$\spp_{d,c}(A,\beta)$ and
$\spp_{d,c}(\cdot,\beta)$
coincide with the problems
$\spp_{d,d}(A,\beta)$ and
$\spp_{d,d}(\cdot,\beta)$,
respectively, whose complexity was shown to be coNP-complete \cite{tw08a}.
For the same reason, problems 
$\spp_{c,d}(A,\beta)$ and
$\spp_{c,d}(\cdot,\beta)$
coincide with 
$\spp_{c,c}(A,\beta)$ and
$\spp_{c,c}(\cdot,\beta)$.
Thus, to complete the complexity picture for problems 
$\spp_{\delta,\epsilon}(\alpha,\beta)$, it suffices to focus on
$\spp_{c,d}(A,\beta)$ and $\spp_{c,d}(\cdot,\beta)$.

First, we prove an upper bound on the complexity of the problem 
$\spp_{c,d}(\cdot,\cdot)$. The proof depends on two lemmas.

\begin{lemma}
\label{newLemma1}
Let $P$ be a program and $A$ and $Y$ sets of atoms. Then, $Y\in
\Mod_\ac(P)$ if and only if $Y'\in\Mod_\ac(P)$, where $Y'=Y\cap(\At(P)\cup A)$.
\end{lemma}
\begin{proof}
{First, we note that atoms that do not occur in $P$ have no effect
on whether an interpretation satisfies the body of a rule in $P$.}
Thus, $T_P(Y)=T_P(Y')$. If $Y\in\Mod_\ac(P)$, then
$Y\models P$ and $Y\setminus T_P(Y)\subseteq\ac$. The former
property implies that $Y'\models P$ 
{(as before, atoms that do not 
occur in $P$ have no effect on whether an interpretation is a model of $P$
or not)}. Since $Y'\setminus T_P(Y')=
Y'\setminus T_P(Y)\subseteq Y\setminus T_P(Y)$, the latter one implies
that $Y'\setminus T_P(Y')\subseteq\ac$. Thus, $Y'\in \Mod_\ac(P)$.

Conversely, let $Y'\in \Mod_\ac(P)$. Then $Y'\models P$ and,
consequently, $Y\models P$ (by the comment made above). Moreover, 
we also have $Y'\setminus
T_P(Y')\subseteq\ac$. Let $y\in Y\setminus T_P(Y)$. If $y\notin Y'$
then, as $y\in Y$ and $Y'=Y\cap(\at(P)\cup A)$, $y\notin A$, that is,
$y\in\ac$. If $y\in Y'$, then $y\in
Y'\setminus T_P(Y')$ (we recall that $T_P(Y)=T_P(Y')$). Hence, $y\in\ac$
in this case, too. It follows that $Y\setminus T_P(Y)\subseteq\ac$
and so, $Y\in\Mod_\ac(P)$.
\end{proof}

\begin{lemma}
\label{newLemma2}
Let $P$ and $Q$ be programs and $A$ a set of atoms. Then, $\Mod_\ac(P)
\not=\Mod_\ac(Q)$ or, for some $Y\in\Mod_\ac(P)$, $T_P(Y)\not=T_Q(Y)$ 
if and only if there is $Y'\subseteq\at(P\cup Q)\cup A$ such that $Y'$ 
belongs to exactly one of $\Mod_\ac(P)$ and $\Mod_\ac(Q)$, or $Y'$ 
belongs to both $\Mod_\ac(P)$ and $\Mod_\ac(Q)$ and $T_P(Y')\not=T_Q(Y')$.
\end{lemma}
\begin{proof}
Clearly, we only need to prove
the ``only-if'' implication. To this end, we note that if
$\Mod_\ac(P)\not=\Mod_\ac(Q)$, then by Lemma \ref{newLemma1},
there is $Y'\subseteq \at(P\cup Q)\cup A$ with that property. Thus,
let
us
assume that
$\Mod_\ac(P)=\Mod_\ac(Q)$. If for some $Y\in\Mod_\ac(P)$,
$T_P(Y)\not=T_Q(Y)$ then again
by the argument given above, 
$Y'=Y \cap(\at(P\cup Q)\cup A)$ belongs to both $\Mod_\ac(P)$ and
$\Mod_\ac(Q)$, and $T_P(Y')=T_P(Y)\not=T_Q(Y)=T_Q(Y')$.
\end{proof}

\begin{theorem}
\label{m-revsupp}
The problem $\spp_{c,d}(\cdot,\cdot)$
is in the class coNP.
\end{theorem}
{\normalfont\itshape Proof}\\
It is sufficient to show that $\spp_{c,d}(\cdot,\emptyset)$ is in coNP, since
$(P,Q,A)$ is a YES instance of $\spp_{c,d}(\cdot,\emptyset)$ if and only if
$(P,Q,A,B)$ is a YES instance of $\spp_{c,d}(\cdot,\cdot)$ (cf.\ Theorem~\ref{nlp1}).

{Thus, we will now focus on proving that $\spp_{c,d}(\cdot,\emptyset)$
is in coNP.} 
Theorem \ref{nlp1} and Lemma \ref{newLemma2} imply the correctness of 
the following algorithm to decide the complementary problem to 
$\spp_{c,d}(\cdot,\emptyset)$ for an instance $(P,Q,A)$: 

\begin{enumerate}
\item nondeterministically guess $Y\subseteq \at(P\cup Q)\cup A$, and
\item verify that $Y$ belongs to exactly one of $\Mod_\ac(P)$ and
$\Mod_\ac(Q)$, or that $Y$ belongs to $\Mod_\ac(P)$ and
$\Mod_\ac(Q)$, and that $T_P(Y)\not=T_Q(Y)$. 
\end{enumerate}

Checking $Y\models P$ and $Y\models Q$ can be done in
polynomial time (in the size of the input, which is given by 
$|At(P\cup Q)\cup A|$). Similarly, for $R=P$ or $Q$, $Y\setminus T_R(Y)
\subseteq\ac$ if and only if $(Y\setminus T_R(Y))\cap A=\emptyset$. 
Thus, checking $Y\setminus T_R(Y)\subseteq\ac$ can be done in
polynomial time, too, and so the algorithm runs in polynomial time.
Hence, the complementary problem to $\spp_{c,d}(\cdot,\emptyset)$ is in 
NP. It follows that the problem $\spp_{c,d}(\cdot,\emptyset)$ is in coNP and
so, the assertion follows.
\hspace*{1em}\proofbox

\smallskip
For the lower bound we use the problem $\spp_{c,d}(A,B)$. Let us comment
that the reduction, from the satisfiability problem to $\spp_{c,d}(A,B)$,
used in the following hardness proof, is indeed computable in polynomial
time. The same is true for the reductions used in all other places in 
the paper. In each case, the polynomial-time computability of the 
reductions is evident, and we do not state it explicitly in the proofs.

\begin{theorem}
\label{h-revsupp}
The problem $\spp_{c,d}(A,B)$ is
coNP-hard. 
\end{theorem}
\begin{proof}
Let us consider a CNF formula $\vph$,\footnote{{Here and throughout
the paper, by \emph{CNF formula} we mean a formula in the conjunctive 
normal form.}} let $Y$ be the set of
atoms in $\vph$, and let $Y'=\{y'\st y\in Y\}$ be a set of new atoms. 
We define
\begin{eqnarray*}
P(\vph)&=&\{ y \lar \nt y';\; y'\lar \nt y;\; \lar y,y'\mid y\in Y\}\cup\\
&& \{ \lar \hat{c} \mid c\mbox{\ is a clause in\ }\vph\}
\end{eqnarray*}
where, for each clause $c\in \varphi$, say $c=y_1\vee\cdots \vee
y_k\vee\neg  y_{k+1}\vee\cdots\vee\neg y_m$, $\hat{c}$ denotes the
the sequence $y'_1,\ldots,y'_k,y_{k+1},\ldots,y_m$. 
To simplify the notation, we write $P$ for $P(\vph)$.
One can check that $\vph$ has a model if and only if $P$ has a
model. Moreover, for every model $M$ of $P$ such that
$M\subseteq\at(P)$,
$M$ is 
a \emph{supported} model of $P$ and, consequently, satisfies
$M=T_{P}(M)$.

Next, let $Q$ consist of $f$ and $\lar f$. As 
$Q$ has no models, Theorem \ref{nlp1} implies that 
$Q$ is supp-equivalent to $P$ relative
to $\caHB(\ac,B)$ if and only if $\Mod_\ac(P)=\emptyset$. 
If $M\in
\Mod_\ac(P)$, then there is $M'\subseteq\At(P)$ such that $M'\in
\Mod_\ac(P)$. Since every model $M'$ of $P$ such that $M'\subseteq \At(P)$
satisfies $M'=T_{P}(M')$, it follows that $\Mod_\ac(P)= \emptyset$ if and
only if $P$ has no models.
Thus, $\vph$ is unsatisfiable if and only if $Q$ is supp-equivalent to
$P$ relative to $\caHB(\ac,B)$, and the assertion follows.
\end{proof}

\smallskip
{The observations made at the beginning of this section, Theorems 
\ref{m-revsupp} and \ref{h-revsupp}, and the relations depicted in
Figure \ref{fig1} imply the following corollary.}

\begin{corollary}\label{cor:supp}
The problem $\spp_{\delta,\varepsilon}(\alpha,\beta)$ is coNP-complete, for
any combination of $\delta,\varepsilon\in\{c,d\}$, $\alpha\in \{A,\cdot\}$, $\beta\in\{B,\cdot\}$.
\end{corollary}

\section{Suppmin-equivalence}

In this section, we establish the complexity for direct-complement,
complement-direct and complement-complement problems of deciding 
suppmin-equivalence. The complexity of direct-direct problems is 
already known \cite{tw08a}. 

\subsection{Upper bounds}
The argument consists of a series of auxiliary results. 
The first two lemmas are
concerned with the basic problem of deciding whether 
$(X,Y)\in\Mod^{B'}_{A'}(P)$, where $A'$ and $B'$ stand for $A$ or $\ac$
and $B$ or $\bc$, respectively.

\begin{lemma}
\label{lem:memb1}
The following problems are in the class coNP: 
Given a program $P$, and sets $X$, $Y$, $A$, and $B$, 
decide whether 
\begin{enumerate}
\item[i.] $(X,Y)\in\Mod^B_{A^c}(P)$; 
\item[ii.] $(X,Y)\in\Mod^\bc_A(P)$;
\item[iii.] $(X,Y)\in\Mod^\bc_\ac(P)$.
\end{enumerate}
\end{lemma}
\smallskip
{\normalfont\itshape Proof}\\
We first show that the complementary problem, this is, to decide 
whether $(X,Y)\notin\Mod^B_\ac(P)$, is in NP. To this end, we observe
that $(X,Y)\notin\Mod^B_\ac(P)$ if and only if 
at least one of the following conditions holds:
\begin{enumerate}
\item 
$Y\notin\Mod_\ac(P)$,
\item 
$X\not\subseteq Y|_{\ac\cup B}$,
\item
there is $Z\subset Y$ such that $Z|_{\ac\cup B}=Y|_{\ac\cup B}$
and $Z\models P$,
\item 
there is $Z\subset Y$ such that $Z|_B=X|_B$, $Z|_\ac \supseteq 
X|_\ac$ and $Z\models P$, 
\item 
$X|_B=Y|_B$ and $Y\setminus T_P(Y)\not\subseteq X$.
\end{enumerate} 
We note that verifying any condition involving $\ac$ can be reformulated
in terms of $A$. For instance, for every set $V$, we have $V|_\ac=V
\setminus A$, and $V\subseteq\ac$ if and only if $V\cap
A=\emptyset$. Thus, the conditions (1), (2) and (5) can be decided in 
polynomial time. 
Conditions (3) and (4) can be decided by a 
nondeterministic polynomial time algorithm. Indeed, once we 
nondeterministically guess $Z$, all other tests can be decided in 
polynomial time.
The proofs for the remaining two claims 
use the same ideas and
differ only in technical details depending on which of $A$ and $B$ is 
subject to the complement operation.
\hspace*{1em}\proofbox

\begin{lemma}
\label{lem:memb2}
For  every finite set $B\subseteq\at$, the following problems are in 
the class $\Pol$:
Given a program $P$, and sets $X$, $Y$, and $A$, decide whether 
\begin{enumerate}
\item[i.] $(X,Y)\in\Mod^\bc_\ac(P)$;
\item[ii.] $(X,Y)\in\Mod^\bc_A(P)$.
\end{enumerate}
\end{lemma} 
\begin{proof}
In each case, the argument follows the same lines as that for Lemma
\ref{lem:memb1}. The difference is in the case of the conditions (3) and
(4). Under the assumptions of this lemma, 
they can
be decided in \emph{deterministic} polynomial time. Indeed, let us note
that there are no more than $2^{|B|}$ sets $Z$ such that
$Z|_{\ac\cup\bc}=Y|_{\ac\cup\bc}$ (or, for the second problem,
such that $Z|_{A\cup\bc}=Y|_{A\cup\bc}$). Since $B$ is
finite, fixed, 
{and not a part of the input,} the condition 
(3) can be checked in polynomial time 
by a simple enumeration of all possible sets $Z$ such that $Z\subset Y$  
and $Z|_{\ac\cup\bc}=Y|_{\ac\cup\bc}$ and checking for each of them
whether $Z\models P$. For the condition (4), the argument is similar. 
Since $Z$ is constrained by $Z|_\bc=X|_\bc$, there are
no more than $2^{|B|}$ possible candidate sets $Z$ to consider in this
case, too. 
\end{proof}

\smallskip
The role of the next lemma is to show that $(X,Y)\in\Mod_A^B(P)$ implies
constraints on $X$ and $Y$.

\begin{lemma}
\label{in0}
Let $P$ be a 
program and $A,B\subseteq\at$. If $(X,Y)\in
\Mod_A^B(P)$ then $X\subseteq Y\subseteq\at(P)\cup A$.
\end{lemma}
\begin{proof}
We have $Y\in\Mod_A(P)$. Thus, $Y\setminus T_P(Y)\subseteq A$
and, consequently, $Y\subseteq T_P(Y)\cup A\subseteq \at(P)\cup A$.
We also have $X\subseteq Y|_{A\cup B}\subseteq Y$. 
\end{proof}

\begin{theorem}
\label{thm:memb1a}
The problem $\sppm_{d,c}(\cdot,\cdot)$ is in the class 
$\PiP{2}$. The problem $\sppm_{d,c}(\cdot,B)$ is in the
class coNP.
\end{theorem}
\begin{proof}
We start with an argument for the problem $\sppm_{d,c}(\cdot,\cdot)$.
By Theorem \ref{thm:general}, $P$ and $Q$ are not suppmin-equivalent
relative to $\caHB(A,\bc)$ if and only if there is $(X,Y)\in \Mod_A^\bc(P)
\div\Mod_A^\bc(Q)$, or there is $(X,Y)\in\Mod_A^\bc(P)$ and $T_P(Y)|_\bc\not=
T_Q(Y)|_\bc$. Thus, by Lemma \ref{in0}, to decide that $P$ and $Q$ are not
suppmin-equivalent relative to $\caHB(A,\bc)$, one can guess $X$ and $Y$
such that $X\subseteq Y\subseteq \at(P\cup Q) \cup A$ and verify that
$(X,Y)\in \Mod_A^\bc(P) \div\Mod_A^\bc(Q)$, or that $(X,Y)\in\Mod_A^\bc(P)$
and $T_P(Y)|_\bc\not=T_Q(Y)|_\bc$. By Lemma \ref{lem:memb1}(ii), deciding
the membership of $(X,Y)$ in $\Mod_A^\bc(P)$ and $\Mod_A^\bc(Q)$ can be
accomplished by means of two calls to a coNP oracle. Deciding 
$T_P(Y)|_\bc\not=T_Q(Y)|_\bc$ can be accomplished in polynomial time
(we note that $T_P(Y)|_\bc=T_P(Y)\setminus B$ and $T_Q(Y)|_\bc=T_Q(Y)
\setminus B$). The argument for the second part of the assertion is
essentially the same. The only difference is that we use Lemma 
\ref{lem:memb2}(ii) instead of Lemma \ref{lem:memb1}(ii) to obtain a stronger
bound.
\end{proof}

\smallskip
Lemma \ref{in0} is too weak for the membership results for
complement-direct and comp\-le\-ment-complement problems. Indeed,
for these two types of problems, it only limits $Y$ to subsets of 
$\at(P)\cup\ac$, which is infinite.
{To handle these two classes of problems we use results that
provide stronger limits on $Y$ and can be used in proofs of the
membership results.}
The 
proofs are quite technical. To preserve the overall flow of the argument,
we present them in the appendix.

\begin{lemma}
\label{lem:new}
Let $P,Q$ be programs and $A,B\subseteq\at$.
\begin{enumerate}
\item 
If $(X,Y)\in\Mod_\ac^B(P)\setminus\Mod_\ac^B(Q)$ then
there is $(X',Y')\in\Mod_\ac^B(P)\setminus\Mod_\ac^B(Q)$ such that 
$Y'\subseteq\At(P\cup Q)\cup A$.
\item 
If $(X,Y)\in\Mod_\ac^B(P)$ and $T_P(Y)|_B\not=T_Q(Y)|_B$,
then there is $(X',Y')\in\Mod_\ac^B(P)$ such that $T_P(Y')|_B\not=
T_Q(Y')|_B$ and $Y'\subseteq\At(P\cup Q)\cup A$.
\end{enumerate}
\end{lemma}

\begin{theorem}
\label{thm:memb1b}
The problems
$\sppm_{c,d}(\cdot,\cdot)$ and
$\sppm_{c,c}(\cdot,\cdot)$
are contained in the class $\PiP{2}$.
The problem
$\sppm_{c,c}(\cdot,B)$ is in the class coNP.
\end{theorem}
\begin{proof}
The argument is similar to that of Theorem \ref{thm:memb1a}.
First, we will consider the problem $\sppm_{c,d}(\cdot,\cdot)$.
By Theorem \ref{thm:general},
$P$ and $Q$ are not suppmin-equivalent relative to
$\caHB(\ac,B)$ if and only if there is $(X,Y)\in \Mod_\ac^B(P)\div
\Mod_\ac^B(Q)$, or $(X,Y)\in\Mod_\ac^B(P)$ and $T_P(Y)|_B\not=T_Q(Y)|_B$.
By Lemma \ref{lem:new}, $P$ and $Q$ are not suppmin-equivalent relative
to $\caHB(\ac,B)$ if and only if there is $(X,Y)$ such that $X\subseteq
Y\subseteq\at(P\cup Q)\cup A$ and $(X,Y)\in\Mod_\ac^B(P)\div
\Mod_\ac^B(Q)$, or $(X,Y)\in\Mod_\ac^B(P)$ and $T_P(Y)|_B\not=
T_Q(Y)|_B$.

Thus, to decide the complementary problem, it suffices to guess
$X,Y\subseteq \At(P\cup Q)\cup A$ and check that
$(X,Y)\in\Mod^B_{A^c}(P)\div\Mod^B_{A^c}(Q)$,
or that
$(X,Y)\in\Mod^B_{A^c}(P)$
and $T_P(Y)|_B \neq T_Q(Y)|_B$.
The 
first task can be decided by NP oracles (Lemma
\ref{lem:memb1}(i)), and testing $T_P(Y)|_B \neq T_Q(Y)|_B$ can
be accomplished in polynomial time.

The remaining arguments are similar. To avoid repetitions, we
only list essential differences.
In the case of $\sppm_{c,c}(\cdot,\cdot)$, we use Lemma 
\ref{lem:memb1}(iii).
To obtain a stronger upper bound for $\sppm_{c,c}(\cdot,B)$, we use 
Lemma~\ref{lem:memb2}(i) instead of Lemma \ref{lem:memb1}(iii).
\end{proof}

\smallskip
When $A$ is fixed to $\emptyset$, that is, we
have $\ac=At$, which means there is no restriction on atoms in the heads
of rules,
a stronger bound on the 
complexity of the complement-complement and complement-direct problems
can be derived. We first state a key lemma (the proof is in the appendix).

\begin{lemma}\label{cor:usuppmin:2}
Let $P,Q$ be programs and $B\subseteq \at$.
If $\Mod_\at^B(P)\not=\Mod_\at^B(Q)$, then there is
$Y\subseteq\at(P\cup Q)$
such that $Y$ is a model of exactly one of $P$ and $Q$, or there is
$a\in Y$ such that $(Y\setminus\{a\},Y)$
belongs to exactly one of $\Mod_\at^B(P)$ and $\Mod_\at^B(Q)$.
\end{lemma}

\begin{theorem}
\label{thm:memb1}
The problems 
$\sppm_{c,c}(\emptyset,\cdot)$
and
$\sppm_{c,d}(\emptyset,\cdot)$
are in the class coNP.
\end{theorem}
\begin{proof}
The case of $\sppm_{c,d}(\emptyset,\cdot)$
was settled before by 
\citeNS{tw08a} (they
denoted the problem by $\sppm_\at$).
Thus, we consider only the problem 
$\sppm_{c,c}(\emptyset,\cdot)$. We will show that the 
following nondeterministic algorithm verifies, given programs 
$P$, $Q$ and a set $B\subseteq\At$, that $P$ and $Q$
are not suppmin-equivalent relative to 
$\caHB(\at,\bc)$.
We guess a
pair $(a,Y)$, where $Y\subseteq \at(P\cup Q)$, and $a\in\at(P\cup Q)$
such that
(a) $Y$ is a model of exactly one of
$P$ and $Q$; or
(b) $a\in Y$ and $(Y\setminus\{a\},Y)$ belongs to exactly
one of $\Mod_\at^\bc(P)$ and $\Mod_\at^\bc(Q)$; or
(c) $Y$ is model of $P$ and $T_P(Y)\setminus B\neq T_Q(Y)
\setminus B$.

Such a pair exists if and only if $P$ and $Q$ are not
suppmin-equivalent relative to 
$\caHB(\at,\bc)$.
Indeed, let us assume that
such a pair $(a,Y)$ exists. If (a) holds for $(a,Y)$, say $Y$ is a model
of $P$ but not of $Q$, then $(Y,Y)\in\Mod_\at^\bc(P)\setminus \Mod_\at^\bc(Q)$
(easy to verify from the definition of $\Mod_\at^\bc(\cdot)$). Thus, 
$\Mod_\at^\bc(P)\not=\Mod_\at^\bc(Q)$ and, by Theorem \ref{thm:general}, 
$P$ and $Q$ are not suppmin-equivalent relative to 
$\caHB(\at,\bc)$.
If 
(b) holds for $(a,Y)$, $\Mod_\at^B(P)\not=\Mod_\at^B(Q)$ again, and we 
are done, as above, by Theorem \ref{thm:general}. Finally, if (c) holds,
$(Y,Y)\in\Mod_\at^\bc(P)$ (as $Y\models P$) and
$T_P(Y)|_\bc= T_P(Y)\setminus B\neq T_Q(Y) \setminus B = T_Q(Y)|_\bc$.
Thus, one more time by Theorem \ref{thm:general}, $P$ and $Q$ are not 
suppmin-equivalent relative to 
$\caHB(\at,\bc)$.

Conversely, if $P$ and $Q$ are not suppmin-equivalent relative to
$\caHB(\at,\bc)$, then $\Mod_\at^\bc(P)\not=\Mod_\at^\bc(Q)$, or 
there is $(X,Y)\in \Mod_\at^\bc(P)$ such that $T_P(Y)|_\bc\neq T_Q(Y)|_\bc$.
By Lemma \ref{cor:usuppmin:2}, if $\Mod_\at^\bc(P)\not=\Mod_\at^\bc(Q)$
then 
there is 
$(a,Y)$ such that $Y\subseteq \at(P\cup Q)$ and $(a,Y)$ satisfies (a)
or (b). Thus, let us assume that there is
$(X,Y)\in \Mod_\at^\bc(P)$ such that $T_P(Y)|_\bc\neq T_Q(Y)|_\bc$.
Then, $Y\models P$ and $T_P(Y)|_\bc\neq T_Q(Y)|_\bc$
or, equivalently, $T_P(Y)\setminus B\neq T_Q(Y)\setminus B$.
Let $Y'=Y\cap \at(P\cup Q)$. Clearly, $Y'\models P$, $T_P(Y)=T_P(Y')$, 
and $T_Q(Y)=T_Q(Y')$. Thus, $T_P(Y')\setminus B\neq T_Q(Y')\setminus B$.
Picking any $a\in\at(P\cup Q)$ (since $P$ and $Q$ are not 
suppmin-equivalent relative to 
$\caHB(\at,\bc)$, $\at(P\cup Q)\not=
\emptyset$) yields a pair $(a,Y')$, with $Y'\subseteq \at(P\cup Q)$, 
for which (c) holds.

It follows that the algorithm is correct. Moreover, checking whether
$Y\models P$ and $Y \models Q$ can clearly be done in polynomial time
in the total size of
$P$, $Q$, and $B$; the same holds 
for checking $T_P(Y)\setminus B\neq T_Q(Y)\setminus B$.
Finally, testing $(Y\setminus \{a\},Y)\in\Mod_\at^\bc(P)$ and 
$(Y\setminus\{a\},Y)\in\Mod_\at^\bc(Q)$ are polynomial-time tasks 
(with respect to the size of the input), too. The conditions
(1) - (3) and (5) are evident. To verify the condition (4), we need to 
verify that $Z\not\models P$ for just one set $Z$, namely $Z=Y\setminus
\{a\}$. Thus, the algorithm runs in polynomial time.  It follows that 
the complement of our problem is in the class NP. 
\end{proof}

\subsection{Lower bounds and exact complexity results}

We start with direct-complement problems. 

\begin{theorem}
\label{thm:suppBAc1}
The problem $\sppm_{d,c}(A,\cdot)$ is $\PiP{2}$-hard.
\end{theorem}
\smallskip
{\normalfont\itshape Proof}\\
Let $\forall Y\exists X \varphi$ be a QBF, where $\varphi$ is a
CNF formula over $X\cup Y$. We can assume that $A\cap X=\emptyset$
(if not, variables in $X$ can be renamed). Next, we can assume that
$A\subseteq Y$ (if not, one can add to $\vph$ ``dummy'' clauses $y\vee \neg y$, for $y\in Y$).
%
We will construct programs $P(\varphi)$ and $Q(\varphi)$, and a set $B$,
so that $\forall Y\exists X\varphi$ is true if and only if $P(\vph)$ and 
$Q(\vph)$ are suppmin-equivalent relative to $\caHB(A,B^c)$. Since the
problem to
decide whether a given QBF $\forall Y\exists X \varphi$ is true is
$\PiP{2}$-complete, the assertion will follow.
 
For every atom $z\in X\cup Y$, we introduce a fresh atom
$z'$ (in particular, in such a way that $z'\notin A$). Given a set of ``non-primed'' atoms $Z$, we define $Z'=\{z'\st z\in
Z\}$. Thus, we have $A\cap (Y'\cup X')=\emptyset$. 
%
We use $\hat{c}$ as in the proof of Theorem~\ref{h-revsupp} and
define the following programs:
\begin{eqnarray*}
P(\varphi) & = & \{ z \lar \nt z';\;  z' \lar \nt z \mid z\in X\cup Y\}\cup
\{ \lar y,y' \mid y \in Y \}\cup \\
&& \{ x \lar u,u';\; x'\lar u,u' \mid x,u\in X \}\cup\\
&& \{ x \lar \hat{c};\; x' \lar \hat{c} \mid x\in X, c\mbox{\ is a clause in\ } \varphi\};\\
Q(\varphi) & = & \{ z \lar \nt z';\;  z' \lar \nt z \mid z\in X\cup Y\}\cup
\{ \lar z,z' \mid z \in X\cup Y \}\cup \\
&& \{ \lar \hat{c} \mid c\mbox{\ is a clause in\ } \varphi \}.
\end{eqnarray*}
To simplify notation, from now on we write $P$ for $P(\vph)$ and $Q$
for $Q(\vph)$. We also define $B=X\cup X'\cup Y\cup Y'$. We observe that 
$\At(P)=\At(Q)=B$.

One can check that the models of $Q$ contained in $B$
are sets of type
\begin{enumerate}
\item
$I\cup (Y\setminus I)'\cup J \cup (X\setminus J)'$,
where $J\subseteq X$, $I\subseteq Y$ and $I\cup J\models \varphi$.
\end{enumerate}
Each model of $Q$ is also a model of $P$ but $P$ has additional models
contained in $B$, viz. 
\begin{enumerate}
\item[2.]
$I \cup (Y\setminus I)' \cup X\cup X'$,
for \emph{each} $I\subseteq Y$.
\end{enumerate}
Clearly, for each model $M$ of $Q$ such that $M\subseteq B$, $T_Q(M)=M$.
Similarly, for each model $M$ of $P$ such that $M\subseteq B$, $T_P(M)=M$.
Hence, each such model $M$ is also supported for both $P$ and $Q$.

From these comments, it follows that for every model $M$ of $Q$ ($P$, respectively),
$T_Q(M)=M\cap B$ ($T_P(M)=M\cap B$, respectively). Thus, for every
model $M$ of both $P$ and $Q$, $T_Q(M)|_{B^c}=T_P(M)|_{B^c}$.
It follows that
$P$ and $Q$ are suppmin-equivalent 
with respect to $\caHB(A,\bc)$
if and only if $\Mod_A^{B^c}(P)=
\Mod_A^{B^c}(Q)$ (indeed, we recall that if $(N,M)\in\Mod_A^{B^c}(R)$ then
$M$ is a model of $R$).

Let us assume that $\forall Y\exists X \varphi$ is false. Hence, there
exists an assignment $I\subseteq Y$ to atoms $Y$ such that for every
$J\subseteq X$, $I\cup J\not\models \varphi$. Let $N=I
\cup(Y\setminus I)'\cup X\cup X'$.
We will show that $(N|_{A\cup {B^c}},N)\in\Mod_A^{B^c}(P)$.

Since $N$ is a supported model of $P$, $N\in\Mod_A(P)$. The requirement
(2) for $(N|_{A\cup {B^c}},N)\in\Mod_A^{B^c}(P)$ is evident. The
requirement (5) holds, since $N\setminus T_P(N)=\emptyset$. By the
property of $I$, $N$ is a minimal model of $P$. Thus, the requirements
(3) and (4) hold, too. It follows that $(N|_{A\cup {B^c}},N)\in\Mod_A^{B^c}(P)$,
as claimed. Since $N$ is not a model of $Q$, $(N|_{A\cup {B^c}},N)\notin
\Mod_A^{B^c}(Q)$.

Let us assume that $\forall Y\exists X \varphi$ is true. First, 
observe that $\Mod_A^{B^c}(Q)\subseteq \Mod_A^{B^c}(P)$. Indeed, let
$(M,N)\in\Mod_A^{B^c}(Q)$. It follows that $N$ is a model of $Q$ and,
consequently, of $P$. From our earlier comments, it follows that
$T_Q(N)=T_P(N)$. Since $N\setminus T_Q(N)\subseteq A$, $N\setminus T_P(N)
\subseteq A$. Thus, $N\in\Mod_A(P)$.  Moreover, if $M|_{B^c}=N|_{B^c}$
then $N\setminus T_Q(N)\subseteq M$ and, consequently, $N\setminus T_P(N)
\subseteq M$. Thus, the requirement (5) for $(M,N)\in\Mod_A^{B^c}(P)$ holds.
The condition $M\subseteq N|_{A\cup {B^c}}$ is evident (it holds as $(M,N)
\in\Mod_A^{B^c}(Q)$). Since $N$ is a model of $Q$, $N=N'\cup V$, where $N'$
is 
a model of type 1
and $V\subseteq \At\setminus B$. Thus,
every model $Z\subset N$ of $P$ is also a model of $Q$. It implies that
the requirements (3) and (4) for $(M,N)\in\Mod_A^{B^c}(P)$ hold. Hence,
$(M,N)\in\Mod_A^{B^c}(P)$ and, consequently, $\Mod_A^{B^c}(Q)\subseteq
\Mod_A^{B^c}(P)$.

We will now use the assumption that $\forall Y\exists X \varphi$ is true
to prove the converse inclusion, i.e., $\Mod_A^{B^c}(P)\subseteq \Mod_A^{B^c}(Q)$.
To this end, let us consider $(M,N)\in
\Mod_A^{B^c}(P)$. If $N=N'\cup V$, where $N'$ is of 
type 1
and $V\subseteq \At\setminus B$, then arguing as above, one can show that
$(M,N)\in\Mod_A^{B^c}(Q)$. Therefore, let us assume that $N=N'\cup V$, where
$N'$ is of 
type 2 and $V\subseteq \At\setminus B$.
More specifically, let $N'=I\cup(Y\setminus I)'\cup X\cup X'$, for some $I\subseteq Y$. By our
assumption, there is $J\subseteq X$ such that $I\cup J\models \vph$.
It follows that
$Z=I\cup(Y\setminus I)'\cup J\cup (X\setminus J)'\cup V$ is a model
of $P$. Clearly, $Z\subset N$. Moreover, 
since $B^c\cap (X\cup X'\cup Y\cup Y')=A\cap (X\cup X'\cup Y\cup Y')=\emptyset$,
we have $Z|_{A\cup {B^c}}=N|_{A\cup {B^c}}$. Since $(M,N)\in\Mod_A^{B^c}(P)$,
the requirement (3) implies that $Z$ is not a model of $P$, a
contradiction. Hence, the latter case is impossible and $\Mod_A^{B^c}(P)
\subseteq\Mod_A^{B^c}(Q)$ follows.

We proved that $\forall Y\exists X \varphi$ is true if and only if
$\Mod_A^{B^c}(P)=\Mod_A^{B^c}(Q)$. This completes the proof of the assertion.
\hspace*{1em}\proofbox

\begin{theorem}\label{thm:suppBAc2}
The problem $\sppm_{d,c}(A,B)$ is coNP-hard.
\end{theorem}
\begin{proof}
Let us consider a CNF formula $\vph$ over a set of atoms $Y$. Without loss of generality we can 
assume that $Y\cap B=\emptyset$. For each atom $y\in Y$, we introduce a
fresh atom $y'$. Thus, in particular, $B\cap (Y\cup Y')=\emptyset$. Finally,
we consider programs $P(\vph)$ and $Q=\{f\lar;\; \lar f\}$ from the proof of 
Theorem~\ref{h-revsupp}.
In the remainder of the proof, we write $P$ for $P(\vph)$.

From the proof of Theorem~\ref{h-revsupp}, we know that $P$ 
has a model
if and only if $\vph$ has a model (is satisfiable).
We will now show that $\Mod^\bc_A(P)\neq\emptyset$ if and only if 
$\vph$ is satisfiable. It is easy to check that $\Mod^\bc_A(Q)=\emptyset$.
Thus, the assertion will follow by Theorem \ref{thm:general}.

Let us assume that $P$ has a model. Then $P$ has a model, say $M$, such that 
$M\subseteq Y\cup Y'$.
We show that $(M,M)\in \Mod^\bc_A(P)$. Indeed,
since $T_P(M)=M$, $M\in\Mod_A(P)$. Also, since $Y\cup Y'
\subseteq \bc$, $M|_{A\cup\bc}
=M$ and so, $M\subseteq M|_{A\cup\bc}$. Lastly, $M\setminus T_P(M)=\emptyset
\subseteq M$. Thus, the conditions (1), (2) and (5) for $(M,M)\in 
\Mod^\bc_A(P)$ hold. Since $M|_{A\cup\bc}=M$ and $M|_{\bc}=M$, there 
is no $Z\subset M$ such that $Z|_{A\cup\bc}= M|_{A\cup\bc}$ or 
$Z|_\bc=M|_\bc$. Thus, also the conditions (3) and (4) hold, 
and
$\Mod^\bc_A(P)\not=\emptyset$ follows.
Conversely, let $\Mod^\bc_A(P)\not=\emptyset$ and let $(N,M)\in
\Mod^\bc_A(P)$. Then $M\in\Mod_A(P)$ and, in particular, $M$ is a model
of $P$.  
\end{proof}

\smallskip
Combining Theorems \ref{thm:suppBAc1} and \ref{thm:suppBAc2} with 
Theorem \ref{thm:memb1a} yields the following result that fully determines the 
complexity of direct-complement problems.

\begin{corollary}\label{cor:suppmindc}
The problems $\sppm_{d,c}(A,\cdot)$ and $\sppm_{d,c}(\cdot,\cdot)$ 
are $\PiP{2}$-complete. The problems 
$\sppm_{d,c}(A,B)$ and $\sppm_{d,c}(\cdot,B)$ 
are coNP-complete.
\end{corollary}

Before we move on to complement-direct and complement-complement
problems, we present a construction that will be of use in both 
cases.  Let $\forall Y\exists X \varphi$ be a QBF, where $\varphi$ is 
a CNF formula over $X\cup Y$. Without loss of generality we can assume that $X$ and $Y$ are
non-empty.

We define $X'$, $Y'$ and $\hat{c}$, for each clause $c$ of $\vph$, as 
before. Next, let $A,B\subseteq\at$ be 
such that: $A\neq\emptyset$, $A\cap(X\cup X'\cup Y\cup Y')=B\cap(X\cup X'
\cup Y\cup Y')=\emptyset$, and let $g\in A$. 

We define $W=X\cup X'\cup Y \cup Y'\cup\{g\}$ and observe that
$X\cup X'\cup Y\cup Y' \subseteq \ac$ and $g\notin\ac$. Finally,
we select an arbitrary element $x_0$ from $X$ and
define the programs $P(\varphi)$ and $Q(\varphi)$ as follows:
\begin{eqnarray*}
P(\varphi) & = &
\{
 \lar \nt y, \nt y';\;
 \lar y,  y' \mid y\in Y\} \cup \\
&&
\{
 \lar u, \nt v, \nt v' ;\;
 \lar u', \nt v, \nt v' \\
&&
\hphantom{\{}
 \lar \nt u, v, v' ;\;
 \lar \nt u',v, v' \mid u,v\in X\} \cup \\
&&
\{
 \lar \hat{c}, x_0, \nt x'_0;\;
 \lar \hat{c}, \nt x_0,  x'_0
\mid c\mbox{\ is a clause in\ } \varphi\} \cup\\
&&
\{
 \lar \nt g \} \cup
\{
u \lar x_0, x'_0, u \mid u \in W \} \\
Q(\varphi) & = &
P(\varphi)\cup 
\{ \lar \nt x_0, \nt x'_0\}.
\end{eqnarray*}

\begin{lemma}
\label{last}
Under the notation introduced above, $\forall Y\exists X\varphi$ is true
if and only if $P(\vph)$ and $Q(\vph)$ are suppmin-equivalent relative to 
$\caHB(A^c,B)$.
\end{lemma}
\begin{proof}
As usual, to simplify notation we write $P$ for $P(\vph)$ and $Q$
for $Q(\vph)$. We observe that $\At(P)=\At(Q)=W$.
We observe that both $P$ and $Q$ have the following models that are 
contained in $W$:\\
\hspace*{0.35in}\parbox{5in}{
\begin{enumerate}
\item \label{eq:top}
$\{g\}\cup X\cup X'\cup I \cup (Y\setminus I)'$,
for \emph{each} $I\subseteq Y$;
and
\item $\{g\}\cup J \cup (X\setminus J)'\cup I\cup (Y\setminus I)'$,
where $J\subseteq X$, $I\subseteq Y$ and $I\cup J\models \varphi$.
\end{enumerate}
}
Moreover, $P$ has also additional models contained in $W$:\\
\hspace*{0.35in}\parbox{5in}{
\begin{enumerate}
\item[3.] $\{g\}\cup I\cup (Y\setminus I)'$,
for \emph{each} $I\subseteq Y$.
\end{enumerate}
}

For each model $M$ of the type \ref{eq:top}, $T_P(M)=T_Q(M)=M$, 
thanks to the rules $u \lar x_0, x'_0,u$, where $u \in W$. Thus, 
for each model $M$ of type \ref{eq:top}, we have
$M\in\Mod_{A^c}(P)$ and $M\in\Mod_{A^c}(Q)$.

Let $M$ be a model of $P$ of one of the other two types. Then, we have
$T_P(M)=\emptyset$. Moreover, since $g\in M$ and $g\notin\ac$,
$M\setminus T_P(M)\not\subseteq\ac$. Thus $M\notin \Mod_\ac(P)$.
Similarly, if $M$ is a model of $Q$ of type 2, 
$T_Q(M)=\emptyset$. For the same reasons as above, $M\notin \Mod_\ac(Q)$.
Hence, $\Mod_{A^c}(P)=\Mod_{A^c}(Q)$, and both $\Mod_\ac(P)$ and
$\Mod_\ac(Q)$ consist of interpretations $N$ of the form $N'\cup V$,
where $N'$ is a set of the type \ref{eq:top} and $V\subseteq\at
\setminus W$. Clearly, for each such set $N$, $T_P(N)=N'=T_Q(N)$.
Thus $T_P(N)|_B=T_Q(N)|_B$ holds for each $(M,N)\in\Mod^B_{A^c}(P)$ (as
$(M,N)\in\Mod^B_{A^c}(P)$ implies
$N\in\Mod_{A^c}(P)$). By Theorem~\ref{thm:general}, it follows
that $P$ and $Q$ are suppmin-equivalent relative to 
$\caHB(A^c,B)$ if and
only if $\Mod^B_{A^c}(P)=\Mod^B_{A^c}(Q)$.

Thus, to complete the proof, it suffices to show that 
$\forall Y\exists X\varphi$ is true if and only if 
$\Mod^B_{A^c}(P)=\Mod^B_{A^c}(Q)$. 

Let us assume that $\forall Y\exists X \varphi$ is false. Hence, there
exists an assignment $I\subseteq Y$ to atoms $Y$ such that for every
$J\subseteq X$, $I\cup J\not\models \varphi$. Let 
$N=\{g\}\cup I\cup(Y\setminus I)'\cup X\cup X'$.
We will show that $(\{g\}|_B,N)\in\Mod_{A^c}^B(Q)$.
Since $N$ is of the type \ref{eq:top}, $N\in\Mod_{A^c}(Q)$. The
requirement (2) for $(\{g\}|_B,N)\in\Mod_{A^c}^B(Q)$ is evident, as $g\in N$. 
The requirement (5) holds, since $N\setminus T_Q(N) =\emptyset\subseteq
\{g\}|_B$. By the property of $I$, $N$ is a minimal model of $Q$. Thus, 
the requirements (3) and (4) hold, too. It follows that $(\{g\}|_B,N)
\in\Mod_{A^c}^B(Q)$, as claimed. On the other hand $(\{g\}|_B,N)\notin
\Mod_{A^c}^B(P)$. Indeed, let $M=\{g\}\cup I\cup (Y\setminus I)'$.
Then $M\models P$ (it is of the type 3).
We now observe that
$M\subset N$, $\{g\}|_B=M|_B$ (as $B\cap (Y\cup Y')=\emptyset$), and
$M|_\ac \supseteq (\{g\}|_B)|_{A^c}$ (as $(\{g\}|_B)|_{A^c}=\emptyset$,
due to the fact that $g\notin\ac$). It follows that $(\{g\}|_B,N)$
violates the condition (4) for $(\{g\}|_B,N)\in\Mod_{A^c}^B(P)$.

Conversely, let us assume that $\forall Y\exists X \varphi$ is true. 
We first observe that $\Mod_{A^c}^B(P)\subseteq \Mod_{A^c}^B(Q)$.  
Indeed, let $(M,N)\in\Mod_{A^c}^B(P)$. Then, $N\in\Mod_{A^c}(P)$ and,
consequently, $N\in\Mod_{A^c}(Q)$. Moreover, if $M|_B=N|_B$, then 
$N\setminus T_P(N)\subseteq M$ and, as $T_P(N)=T_Q(N)$, $N\setminus
T_Q(N)\subseteq M$. Next, as $(M,N)\in\Mod_{A^c}^B(P)$, $M\subseteq 
N|_{{A^c}\cup B}$. Thus, the requirements (1), (5) and (2) for $(M,N)
\in\Mod_{A^c}^B(Q)$ hold. Since every model of $Q$ is a model of $P$,
it follows that the conditions (3) and (4) hold, too.

We will now use the assumption that $\forall Y\exists X \varphi$ is true
to prove the converse inclusion $\Mod_{A^c}^B(Q)\subseteq \Mod_{A^c}^B(P)$.
To this end, let us consider $(M,N)\in
\Mod_{A^c}^B(Q)$. Reasoning as above, we can show that the conditions 
(1), (5) and (2) for $(M,N)\in\Mod_{A^c}^B(P)$ hold. 

By our earlier comments, $N=N'\cup V$, where $N'$ is of the form 
\ref{eq:top} and $V\subseteq \At\setminus W$. More specifically, 
$N'=\{g\}\cup I\cup(Y\setminus I)'\cup X\cup X'$, for some $I\subseteq
Y$. 

Let us consider $Z\subset N$ such that $Z|_{\ac\cup B}=N|_{\ac\cup B}$.
Since $W\setminus \{g\}\subseteq \ac$, $Z\supseteq N|_{\ac\cup B}
\supseteq I\cup(Y\setminus I)'\cup X\cup X'$. It follows that $Z\cap 
W$ is not of the type 3.
Thus, since $Z\not\models Q$,
$Z\not\models P$. Consequently, the condition (3) for $(M,N)\in
\Mod_{A^c}^B(P)$ holds.

So, let us consider $Z\subset N$ such that $Z|_{B}=M|_{B}$ and
$Z|_{\ac}\supseteq M|_{\ac}$. Let us assume that $Z\models P$. Since $Z
\not\models Q$, $Z=Z'\cup U$, where $Z'$ is a set of the type 3
and $U\subseteq\at\setminus W$. Since $Z\subseteq N$,
$Z'\subseteq N'$, and so, $Z'=\{g\}\cup I\cup(Y\setminus I)'$. 

Since $\forall Y\exists X \varphi$ is true, there is $J\subseteq X$ such
that $I\cup J\models \vph$. It follows that 
\[
N''=\{g\}\cup I\cup(Y\setminus I)'\cup J\cup (X\setminus J)'\cup U
\]
is a model of both $P$ and $Q$ (of the type 2). 
Since $B\cap W\subseteq\{g\}$, it follows that $N''|_B= Z|_B=M|_B$.
Since $N''\supseteq Z$, $N''|_\ac\supseteq Z|_\ac\supseteq M|_\ac$.
Moreover, $N''\subset N$. Since $(M,N)\in \Mod_{A^c}^B(Q)$, $N''\not
\models Q$, a contradiction. Thus, $Z\not\models P$ and, consequently,
the condition (4) for $(M,N)\in \Mod_{A^c}^B(P)$ holds. This completes
the proof of $\Mod_{A^c}^B(Q)\subseteq\Mod_\ac^B(P)$ and of the lemma.
\end{proof}

\smallskip
We now apply this lemma to complement-direct problems. We have the following result.

\begin{theorem}
\label{thm:136}
The problem $\sppm_{c,d}(A,B)$, where $A\neq \emptyset$,
is $\PiP{2}$-hard.
\end{theorem}
\begin{proof} Let $\forall Y\exists X \varphi$ be a QBF, where $\varphi$ is a
CNF formula over $X\cup Y$ such that $X$ and $Y$ are nonempty. We can assume 
that $A\cap (X\cup Y)=B\cap 
(X\cup Y)=\emptyset$ (if not, variables in the QBF can be renamed).
We define $X'$ and $Y'$ as in other places. Thus, $(A\cup B)\cap
(X'\cup Y')=\emptyset$. Finally, we pick $g\in A$, and define 
$P(\vph)$ and $Q(\vph)$ as above. By Lemma \ref{last}, $\forall Y\exists X \varphi$ is 
true if and only if $P(\vph)$ and $Q(\vph)$ are suppmin-equivalent with respect to
$\caHB(\ac,B)$. Thus, the assertion follows.
(We note that since $B$ is fixed, we cannot assume $g\in B$ or 
$g\notin B$ here; however, Lemma~\ref{last} takes care of both cases).
\end{proof}

\smallskip
We are now in a position to establish exactly the complexity of 
complement-direct problems.

\begin{corollary}\label{cor:suppmincd}
The problems $\sppm_{c,d}(\cdot,B)$ and  $\sppm_{c,d}(\cdot,\cdot)$ are
$\PiP{2}$-complete.
For $A\neq\emptyset$, 
the problems $\sppm_{c,d}(A,B)$, 
and $\sppm_{c,d}(A,\cdot)$, 
are also $\PiP{2}$-complete.
The problems $\sppm_{c,d}(\emptyset,B)$
and $\sppm_{c,d}(\emptyset,\cdot)$ are coNP-complete.
\end{corollary}
\begin{proof} 
For problems $\sppm_{c,d}(A,B)$ (where $A\not=\emptyset$), 
$\sppm_{c,d}(\cdot,B)$, $\sppm_{c,d}(A,\cdot)$ (where $A\not=\emptyset$),
and $\sppm_{c,d}(\cdot,\cdot)$, the upper bound follows from Theorem
\ref{thm:memb1b}, and the lower bound from Theorem \ref{thm:136}.
The problems $\sppm_{c,d}(\emptyset,B)$ and $\sppm_{c,d}(\emptyset,\cdot)$
were proved to be coNP-complete by 
\citeNS{tw08a} (in fact, they denoted these problems by $\sppm_\at^B$ and 
$\sppm_\at$, respectively).
\end{proof}

\smallskip
We will now apply Lemma \ref{last} to complement-complement problems.

\begin{theorem}
\label{sppmccLB2}
The problem $\sppm_{c,c}(A,\cdot)$, where $A\not=\emptyset$, is 
$\PiP{2}$-hard.
\end{theorem}
\begin{proof} Let $\forall Y\exists X \varphi$ be a QBF, where $\varphi$ is a
CNF formula over $X\cup Y$. We select $g\in A$,
and define $X'$ and $Y'$ as usual. Without loss of generality we can assume that $A\cap (X\cup
X'\cup Y\cup Y')=\emptyset$. In particular, $g\notin X\cup X'\cup Y\cup 
Y'$. We set $B=X\cup X'\cup Y \cup Y'$ and so, $\bc\cap(X\cup X'\cup Y \cup Y')
=\emptyset$. Finally, we set $W=X\cup X'\cup Y \cup Y' \cup\{g\}$ and
define programs $P$ and $Q$ as we did in preparation for Lemma \ref{last}.
By Lemma \ref{last}, $\forall Y\exists X \varphi$ is true if and only if
$P$ and $Q$ are suppmin-equivalent with respect to $\caHB(\ac,\bc)$.
Thus, the assertion follows.
\end{proof}

\smallskip
Next, we determine the lower bound for the problem $\sppm_{c,c}(A,B)$.

\begin{theorem}
\label{sppmccLB3}
The problem $\sppm_{c,c}(A,B)$ is coNP-hard.
\end{theorem}
\begin{proof} The problem $\sppm_{c,c}(\emptyset,\emptyset)$ is coNP-complete
\cite{tw08a} (in the paper proving that fact, the problem was denoted by 
$\sppm_\at^\at$).
We will show that it can be reduced to $\sppm_{c,c}(A,B)$
(for any finite $A,B\subseteq\at$). 

Thus, let us fix $A$ and $B$ as two finite subsets of $\at$, and let $P$
and $Q$ be normal logic programs. We define $P'$ and $Q'$ to be programs
obtained by replacing consistently atoms in $P$ and $Q$ that belong to
$A\cup B$ with atoms that do not belong to $\at(P\cup Q)\cup A\cup B$.
Clearly, $P$ and $Q$ are suppmin-equivalent
relative to $\caHB(\at,\at)$ if and only if $P'$ and $Q'$ are suppmin-equivalent
relative to $\caHB(\at,\at)$.

Moreover, it is clear that 
suppmin-equivalence
relative to $\caHB(\at,\at)$ 
between $P'$ and $Q'$
implies 
suppmin-equivalence
relative to $\caHB(\ac,\bc)$
between $P'$ and $Q'$. We will now show the converse implication.
To this end, let $R$ be an arbitrary program from $\caHB(\at,\at)$.
By $R'$ we denote the program obtained by replacing consistently atoms 
in $R$ that belong to $A\cup B$ with atoms that do not belong to $\at(P'
\cup Q') \cup A\cup B$. Since $P'$ and $Q'$ are suppmin-equivalent
relative to $\caHB(\ac,\bc)$, $P'\cup R'$ and $Q'\cup R'$ have the same
suppmin models. Now, we note that because $(A\cup B)\cap \at(P'\cup Q')=
\emptyset$, $P'\cup R'$ and $Q'\cup R'$ have the same suppmin models if
and only if $P'\cup R$ and $Q'\cup R$ have the same suppmin models. Thus,
$P'\cup R$ and $Q'\cup R$ have the same suppmin models and, consequently,
$P'$ and $Q'$ are  suppmin-equivalent relative to $\caHB(\at,\at)$. It follows
that $P$ and $Q$ are suppmin-equivalent relative to $\caHB(\at,\at)$.

By this discussion $P$ and $Q$ are 
suppmin-equivalent relative to $\caHB(\at,\at)$ if and only if $P'$ and 
$Q'$ are suppmin-equivalent relative to $\caHB(\ac,\bc)$. 
coNP-hardness of $\sppm_{c,c}(A,B)$ thus follows from the coNP-hardness
of $\sppm_{c,c}(\emptyset,\emptyset)$.
\end{proof}

\smallskip
Taking into account Theorems \ref{thm:memb1b} and \ref{thm:memb1},
Theorems \ref{sppmccLB2} and \ref{sppmccLB3}
yield the following result.

\begin{corollary}\label{cor:suppmincc}
The problems $\sppm_{c,c}(A,\cdot)$, with $A\not=\emptyset$,
and $\sppm_{c,c}(\cdot,\cdot)$ are $\PiP{2}$-complete. The problems 
$\sppm_{c,c}(A,B)$, $\sppm_{c,c}(\cdot,B)$, and
$\sppm_{c,c}(\emptyset,\cdot)$ are coNP-complete.
\end{corollary}

\section{Stable-equivalence}
\label{stabeq}
In this section, we establish the complexity 
for direct-complement,
complement-direct and complement-com\-plement problems of 
deciding stable-equivalence. 
We will again make use of the relations depicted in Figure \ref{fig1} 
to obtain our results. Thus, for instance, when we derive an upper bound 
for a problem $\stb_{\delta,\varepsilon}(\cdot,\cdot)$ and a matching 
lower bound for $\stb_{\delta,\varepsilon}(A,B)$, we obtain the exact
complexity result for all problems between $\stb_{\delta,\varepsilon}(A,B)$
and $\stb_{\delta,\varepsilon}(\cdot,\cdot)$ (inclusively).
As we will show, for stable equivalence those bounds match in all cases
other than $\delta=\varepsilon=c$.

We also mention that for the upper bounds for relativized hyperequivalence
with respect to the stable-model semantics, some relevant results were
established before. Specifically, the direct-direct problem 
$\stb_{d,d}(\cdot,\cdot)$ is known to be in the class 
$\PiP{2}$ and, under the restriction to normal logic programs, in coNP
\cite{Woltran07}. However, for the sake of 
completeness we treat the direct-direct problems here in full detail 
as, in the case of fixed alphabets, they were not considered before.

\subsection{Upper Bounds}

The following lemmas mirror the corresponding results from the previous
section but show some interesting differences.
For instance, as the following result shows, the problem of model checking 
is slightly harder now compared to Lemma~\ref{lem:memb1}. Namely, it is 
located in the class D$^P$. (We recall that the class D$^P$
consists of all problems expressible as the conjunction
of a problem in NP and a problem in coNP.) However, this increase 
in complexity
compared to Lemma~\ref{lem:memb1} does not influence the subsequent
$\PiP{2}$-membership results, since a call to a D$^P$-oracle amounts to two NP-oracle calls.

\begin{lemma}
\label{lem:smemb1}
The following problems are in the class
D$^P$:
given a program $P$, and sets $X$, $Y$, $A$, and $B$, 
decide whether 
%
$(X,Y)\in\SE^{B'}_{A'}(P)$, where $A'$ stands for one of $A$ and $\ac$,
and $B'$ stands for one of $B$ and $\bc$,
\end{lemma}
\begin{proof}
We use similar arguments as in the proof of Lemma~\ref{lem:memb1}, but
we need now both an NP and a coNP test.

We recall that verifying any condition involving $\ac$ can be
reformulated in terms of $A$. For instance, for every set $V$,
we have $V|_\ac=V \setminus A$, and $V\subseteq\ac$ if and only
if $V\cap A=\emptyset$. The same holds for $\bc$. 

Let $A'\in\{A,\ac\}$ and $B'\in\{B,\bc\}$. We will use the observation
above to establish upper bounds on the complexity of deciding each of 
the conditions (1) - (5) for $(X,Y)\in\SE^{B'}_{A'}(P)$.

The condition (1) can clearly be decided in polynomial time. The same 
holds for the condition (2). It is evident once we note that $X
\subseteq Y|_{A' \cup B'}$ is equivalent 
to $X\subseteq Y\cap (A \cup B)$,
$X\subseteq (Y\cap B)\cup (Y\setminus A)$,
$X\subseteq (Y\cap A)\cup (Y\setminus B)$,
and $X\subseteq Y\setminus(A\cap B)$, depending on the
form of $A'$ and $B'$.
  
It is also easy to show that each of the conditions (3) and (4) can be 
decided by means of a single coNP test, and that the condition (5) can 
be decided by means of one NP test. For all instantiations of $A'$ and 
$B'$, the arguments are similar. We present the details for one case 
only. For example, if $A'$ stands for $A$ and $B'$ stands for $\bc$, to
decide whether $(X,Y)$ violates the
condition (4), we guess a set $Z \subset 
Y$ and verify that (a)
$Z|_\bc\subseteq X|_\bc$ (by checking that $Z\setminus B\subseteq X\setminus B$);
(b) $X|_A \subseteq Z|_A$;
(c) one of the two inclusions is proper; and
(d) $Z\models P^Y$. 
All these tasks can be accomplished in polynomial time, and so deciding 
that the condition (4) does not hold amounts to an NP test. Consequently,
deciding that the condition (4) holds can be accomplished by a coNP test.
\end{proof}

\smallskip
When we fix $A$ and $B$ (they are no longer components of the input),
the complexity of testing whether $(X,Y)\in
\SE^\bc_\ac(P)$ is lower --- the problem is in the class $\Pol$. Comparing
with Lemma~\ref{lem:memb2}, the lower complexity holds only for $A'=\ac$
and $B'=\bc$. Moreover, \emph{both} $A$ and $B$ must be fixed. 

\begin{lemma}
\label{lem:smemb2}
For every finite sets $A,B\subseteq\at$ the following problem is
in the class $\Pol$: given a program $P$, and sets $X$, $Y$, decide
whether $(X,Y)\in\SE^\bc_\ac(P)$.
\end{lemma} 
\begin{proof}
As we noted, testing the conditions (1) and (2) for $(X,Y)\in
\SE^\bc_{A^c}(P)$ can be done in polynomial time. 

For the condition (3)
we check all candidate sets $Z$. Since $Z|_{\ac} =
Y|_{\ac}$ all elements of $Z$ are determined by $Y$ except possibly
for those that are also in $A$. Thus, there are at most $2^{|A|}$ 
possible sets
$Z$ to consider.
Since $A$ is fixed (not a part of the input), checking for all these
sets $Z$ whether $Z\models P^Y$ and $Z\subset Y$ can be done in 
polynomial time.

For the condition (4), the argument is similar.
We note that $Z$ is, in particular, restricted by $Z|_\bc\subseteq X|_\bc$
and $X|_\ac\subseteq Z|_\ac$. The two conditions imply that
$X|_{\ac\cap\bc}=Z|_{\ac\cap\bc}$. Thus, all elements of $Z$ are 
determined except possibly for those that are also in $A\cup B$. It follows
that there are at most $2^{|A\cup B|}$ possibilities for $Z$ to consider. 
Clearly, for each of them, we can check whether it satisfies or fails 
the premises and the consequent of (4) in polynomial time. Thus, checking
the condition (4) is a polynomial-time task.

The same (essentially) argument works also for the condition (5).
Since 
$Z|_{\ac\cup\bc} = X|_{\ac\cup\bc}$, 
all elements of $Z$ are determined except possibly for those that are
also in $A\cap B$. Thus, there are at most $2^{|A\cap B|}$ possible sets
$Z$ to consider. Given that $A$ and $B$ are fixed, checking all those sets
$Z$ for $Z\models P^Y$ and 
$Z\subset Y$ 
can be done in polynomial time. 
\end{proof}

\smallskip
The reduct of a \emph{normal} program is a Horn program. That property
allows us to obtain stronger upper bounds for the case of normal logic 
programs.

\begin{lemma}
\label{lem:smemb1:normal}
The following problems are in the class
$\Pol$. 
Given a normal program $P$, and sets $X$, $Y$, $A$, and $B$, 
decide whether $(X,Y)\in\SE^{B'}_{A'}(P)$, where $A'$ stands for $A$ or
$\ac$, and $B'$ stands for $B$ or $\bc$. 
\end{lemma}
\begin{proof}
As we noted, deciding the conditions (1) and (2) can be accomplished in
polynomial time (even without the assumption of normality).

To show that the condition (3) can be decided in polynomial time, we 
show that the complement of (3) can be decided in polynomial time. 
The complement of (3) has the form: {there is $Z\subset Y$ such that 
$Z|_{A'} = Y|_{A'}$ and $Z\models P^Y$}. Let us consider the Horn
program
$
P' = P^Y \cup Y|_{A'}
$. 
Since $P$, $Y$ and $A$ are given, $P'$ can be constructed in polynomial
time (for instance, if $A=\ac$, $P'=P^Y\cup (Y\setminus A)$). 
We will show that the complement of the condition (3) holds if and only 
if $P'$ is consistent and its least model, say $L$, satisfies $L\subset
Y$ and
$L|_{A'}=Y|_{A'}$.
First, we observe that if the complement of (3) holds, then $P'$ has a 
model $Z$ such that $Z\subset Y$ and 
$Z|_{A'} = Y|_{A'}$. It follows that $P'$ is consistent and its least
model, say $L$, satisfies $L\subseteq Z$. Thus, $L\subset Y$ and 
$L|_{A'} \subseteq Y|_{A'}$. Moreover, since $L\models P'$, $Y|_{A'}
\subseteq L$. Thus, $Y|_{A'}\subseteq L|_{A'}$. Therefore, we have
$L\subset Y$ and $L|_{A'}=Y|_{A'}$ as needed. 
The converse implication is trivial. Since 
$P'$ can be constructed in polynomial time and
$L$ can be computed in polynomial time ($P'$ is Horn),
deciding the complement of the condition (3) can be accomplished
in polynomial time, too.

To settle the condition (4), we again demonstrate that the complement
of the condition (4) can be decided in polynomial time. To this end,
we observe that the complement of (4) holds if and only if one of the
following two conditions holds:\\
\hspace*{0.35in}\parbox{5in}{
\begin{enumerate}
\item[4$'$.] there is $Z\subset Y$ such that, $X|_{A'}\subseteq Z|_{A'}$,
$Z|_{B'} \subset X|_{B'}$ and $Z\models P^Y$
\item[$4''$.] there is $Z\subset Y$ such that, $X|_{A'} \subset Z|_{A'}$,
$Z|_{B'} \subseteq X|_{B'}$ and $Z\models P^Y$.
\end{enumerate}
}

\smallskip
One can check that (4$'$) holds if and only if $P^Y\cup X|_{A'}$ is
consistent and its least model, say $L$, satisfies $L\subset Y$
and $L|_{B'} \subset X|_{B'}$. Similarly, (4$''$) holds if and only if
there is $y\in (Y\setminus X)|_{A'}$ such that $P^Y\cup
(X\cup\{y\})|_{A'}$ is consistent and its least model, say $L$,
satisfies $L\subset Y$ and $L|_{B'} \subseteq X|_{B'}$. Thus, the
conditions (4$'$) and (4$''$) can be checked in polynomial time.

The argument for the condition (5) is similar to that for the complement
of the condition (3). The difference is that instead of $P'$ we use the
Horn program $P^Y \cup X|_{A'\cup B'}$.
Reusing the argument for (3) with the arbitrary containment of $Z$ in $Y$ 
(rather than a proper one) shows that
the complement of (5) can be decided in polynomial time.
\end{proof}

\smallskip

The next lemma plays a key role
in establishing an upper bound on the complexity of the problems
$\stb_{\delta,\varepsilon}(\cdot,\cdot)$. Its proof is technical
and we present it in the appendix.

\begin{lemma}
\label{stin1e}
Let $P,Q$ be programs and $A,B\subseteq\at$. If $(X,Y)\in\SE_A^B(P)\setminus
\SE_A^B(Q)$, then there are sets $X',Y'\subseteq\at(P\cup Q)$,
such that at least one of the following conditions holds:
\begin{enumerate}
\item[i.] $(X',Y')\in\SE_A^B(P)\setminus \SE_A^B(Q)$
\item[ii.] $A\setminus\at(P\cup Q) \neq \emptyset$ and
for every $y,z\in A\setminus\at(P\cup Q)$, $(X',Y'\cup\{y,z\})\in\SE_A^B(P)
\setminus \SE_A^B(Q)$
\end{enumerate}
\end{lemma}

We now use similar arguments to those in the previous section to 
obtain the following collection of membership results.

\begin{theorem}
\label{thm:memb3}
The problem $\stb_{\delta,\varepsilon}(\cdot,\cdot)$, 
is contained in the class $\PiP{2}$, for any
$\delta,\varepsilon\in\{c,d\}$;
$\stb_{c,c}(A,B)$ is contained in the class coNP.
The problem $\stb^n_{\delta,\varepsilon}(\cdot,\cdot)$, 
is contained in the class coNP for any
$\delta,\varepsilon\in\{c,d\}$.
\end{theorem}
\begin{proof}
Given finite programs $P$  and $Q$, and finite subsets $A,B$ of $\at$
the following algorithm decides the complementary problem to 
$\stb_{\delta,\varepsilon}(\cdot,\cdot)$. If $\delta=d$ and
$A\setminus\at(P\cup Q)=\emptyset$, the algorithm guesses
two sets $X,Y\subseteq \at(P\cup Q)$. It verifies whether
$(X,Y)\in\SE_A^B(P)\div\SE_A^B(Q)$ and if so, returns YES.
Otherwise, the algorithm guesses two sets $X,Y\subseteq \at(P\cup Q)$.
If $\delta=d$, it selects two elements $y,z\in A\setminus\at(P\cup Q)$
or, if $\delta=c$, it selects two elements $y,z\in \ac\setminus\at(P\cup Q)$. 
The algorithm verifies whether $(X,Y)\in\SE_{A'}^B(P)\div\SE_{A'}^B(Q)$
(where $A'=A$ if $\delta=d,$ and $A'=\ac$ if $\delta=c$) and
if so, returns YES.
%
%
Otherwise, 
the algorithm verifies whether 
$(X,Y\cup\{y,z\})\in\SE_{A'}^B(P)\div \SE_{A'}^B(Q)$ 
(where $A'=A$ if $\delta=d,$ and $A'=\ac$ if $\delta=c$) and if so, 
returns YES.

The correctness of the algorithm follows by Lemma \ref{stin1e}. 
Since the sizes of $X$ and $Y$ are polynomial in the size of $P\cup Q$, 
the membership of the complementary problem in the class $\Sigma_2^P$
follows by Lemma \ref{lem:smemb1}. 

The remaining claims of the assertion follow in the same way by Lemmas
\ref{lem:smemb2} and \ref{lem:smemb1:normal}, respectively.
\end{proof}

\subsection{Lower bounds and exact complexity results}

We start with the case of normal programs.

\begin{theorem}\label{thm:norm}
The problem $\stb^n_{\delta,\varepsilon}(A,B)$
is coNP-hard for any $\delta,\varepsilon\in\{c,d\}$.
\end{theorem}
\begin{proof}
Let us fix $\delta$ and $\varepsilon$, and let $A'$ and $B'$ be sets
of atoms defined by the combinations $A$ and $\delta$, and $B$ and
$\varepsilon$. We will show that UNSAT can be reduced to 
$\stb^n_{\delta,\varepsilon}(A,B)$.                               

Let $\varphi$ be a CNF over of set of atoms $Y$. We define $P(\varphi)$ 
and $Q$ as in the proof of Theorem~\ref{h-revsupp}. We note that both
programs are normal. As before, we write $P$ instead of $P(\vph)$ in 
order to simplify the notation.

To prove the assertion it suffices to show that
$\varphi$ is unsatisfiable if and only if $P$ and
$Q$ are stable-equivalent with respect to $\caHB(A',B')$. To this end,
we will show that $\varphi$ is unsatisfiable if and only if 
$\SE^{B'}_{A'}(P)=\SE^{B'}_{A'}(Q)$ (cf. Theorem \ref{thm:stable}).

Since $Q$ has no models, $\SE^{B'}_{A'}(Q)=\emptyset$. Moreover,
$\SE^{B'}_{A'}(P)=\emptyset$ if and only if $P$ has no
models (indeed, if $(X,Y)\in\SE^{B'}_{A'}(P)$, then $Y$ is a model 
of $P$; if $Y$ is a model of $P$, then $(Y,Y)\in\SE^{B'}_{A'}(P)$).
It follows that $\SE^{B'}_{A'}(P)=\SE^{B'}_{A'}(Q)$ if and only if
$P$ has no models.

In the proof of Theorem~\ref{h-revsupp}, we noted that $P$ has models
if and only if $\varphi$ has models. Thus, $\SE^{B'}_{A'}(P)
=\SE^{B'}_{A'}(Q)$ if and only if $\varphi$ is unsatisfiable. 
\end{proof}

\smallskip
Together with the matching coNP-membership results for $\stb^n_{\delta,\varepsilon}(\cdot,\cdot)$
from Theorem~\ref{thm:memb3} we obtain the following result.

\begin{corollary}
\label{cor74}
The following problems 
are coNP-complete for any $\delta,\varepsilon\in\{c,d\}$:
$\stb^n_{\delta,\varepsilon}(\cdot,\cdot)$,
$\stb^n_{\delta,\varepsilon}(A,\cdot)$,
$\stb^n_{\delta,\varepsilon}(\cdot,B)$ and
$\stb^n_{\delta,\varepsilon}(A,B)$.
\end{corollary}

We now turn to the case of disjunctive programs.
It turns out that the problems $\stb_{c,d}(A,B)$,
$\stb_{d,d}(A,B)$ and $\stb_{d,c}(A,B)$ are $\Pi_2^P$-hard.
The situation is different for $\stb_{c,c}(A,B)$.
By Theorems \ref{thm:memb3} and Corollary \ref{cor74},
the problem is coNP-complete. However, the two immediate 
successors of that problem, $\stb_{c,c}(A,\cdot)$ and 
$\stb_{c,c}(\cdot,B)$ (cf. Figure \ref{fig1}) are 
$\Pi_2^P$-hard. We will now show these results.

To start with we provide some technical results concerning the structure
of the set $\SE_A^{B}(P)$ when $\at(P)\subseteq A$ and $\at(P)\cap B
=\emptyset$. It will be applicable to programs we construct below.

\begin{lemma}\label{lemma:ue}
Let $P$ be a program and $A,B\subseteq \At$. If $\at(P)\subseteq A$
and $\at(P)\cap B=\emptyset$, then $(X,Y)\in\SE_A^{B}(P)$ if and only
if there are $X',Y'\subseteq\at(P)$ and $W\subseteq A\setminus \at(P)$
such that one of the following conditions holds:
\begin{enumerate}
\item[a.] $X=X'\cup W$, $Y=Y'\cup W$, and $(X',Y')\in\SE_A^{B}(P)$
\item[b.] $X=X'\cup W$, $(X',X')\in\SE_A^{B}(P)$ and $Y=X'\cup
W\cup \{y\}$, for some $y\in A\setminus\at(P)$
\item[c.] $X=X'\cup W$, $(X',X')\in\SE_A^{B}(P)$ and $Y=X'\cup
W\cup D$, for some $D\subseteq B\cap(A\setminus\at(P))$ such that
$W\cap D=\emptyset$ and $|D| \geq 2$.
\end{enumerate}
\end{lemma}

The proof of this result is technical and we give it in the appendix.
This lemma points to the crucial role played by those pairs $(X,Y)\in
\SE_A^{B}(P)$ that satisfy $Y\subseteq\at(P)$. In particular, as noted in the
next result, it allows
to narrow down the class of pairs $(X,Y)$ that need to be tested for 
the membership in $\SE_A^{B}(P)$ and $\SE_A^{B}(Q)$ when considering
stable-equivalence of $P$ and $Q$ with respect to $\caHB(A,B)$.

\begin{lemma}
\label{lemma:ue:a}
Let $P$ and $Q$ be programs, and $A,B$ subsets of $\at$ such that
$\at(P\cup Q)\subseteq A$ and $\at(P\cup Q)\cap B=\emptyset$. Then,
$P$ and $Q$ are stable-equivalent with respect to $\caHB(A,B)$ if
and only if for every $X,Y$ such that $Y\subseteq\at(P\cup Q)$, $(X,Y)
\in\SE_A^{B}(P)$ if and only if $(X,Y)\in \SE_A^{B}(Q)$. 
\end{lemma}
\begin{proof}
Without loss of generality, we can assume that $\at(P)=\at(Q)$. Indeed,
let $P'=P\cup\{a\leftarrow a \mid a\in \at(Q)\setminus\at(P)\}$ and
$Q'=Q\cup\{a\leftarrow a \mid a\in \at(P)\setminus\at(Q)\}$. It is easy to 
see that $P$ and $P'$ ($Q$ and $Q'$, respectively) are stable-equivalent 
with respect to $\caHB(A,B)$. Thus, in particular, $\SE_A^{B}(P)=
\SE_A^{B}(P')$ and $\SE_A^{B}(Q)=\SE_A^{B}(Q')$. Moreover, $\at(P')=\at(Q')
=\at(P\cup Q)$. Therefore, $\at(P'\cup Q')\subseteq A$ if and only if $\at(P\cup 
Q)\subseteq A$, and $\at(P'\cup Q')\cap B=\emptyset$ if and only if 
$\at(P\cup Q)\cap B=\emptyset$.

Thus, let us assume that $\at(P)=\at(Q)$. Only the ``if'' part of the claim 
requires a proof, the other implication being evident. Let us assume that
$(X,Y)\in\SE_A^{B}(P)$. By Lemma \ref{lemma:ue}, there are $X',Y'\subseteq
\at(P)$ and $W\subseteq A\setminus \at(P)$ such that one of the conditions
(a) - (c) holds. If (a) holds, $(X',Y')\in\SE_A^{B}(Q)$ and so, $(X,Y)\in
\SE_A^{B}(Q)$. If (b) or (c) holds, $(X',X')\in\SE_A^{B}(Q)$ and so,
$(X,Y)\in\SE_A^{B}(Q)$, as well.
\end{proof}

\smallskip
Finally, we note that under the assumptions of Lemma \ref{lemma:ue}, if
$Y\subseteq \at(P)$, then the conditions for $(X,Y)\in\SE_A^{B}(P)$ 
simplify. 

\begin{lemma}
\label{lemma:ue:b}
Let $P$ be a program and $A,B\subseteq \At$. If $\at(P)\subseteq A$,
$\at(P)\cap B=\emptyset$ and $Y\subseteq\at(P)$, then $(X,Y)\in
\SE_A^{B}(P)$ if and only if $Y\models P$, $X\subseteq Y$, $X\models P^Y$,
and for every $Z\subset Y$ such that $X\subset Z$,
$Z\not\models P^Y$.  
\end{lemma}
\begin{proof}
Under the assumptions of the lemma, the four conditions are equivalent to 
the conditions (1), (2), (5) and (4) for  $(X,Y)\in\SE_A^{B}(P)$, 
respectively, and the condition (3) is vacuously true.
\end{proof}

\smallskip
Our first $\PiP{2}$-hardness result for stable equivalence results concerns
the problem $\stb_{c,d}(A,B)$.

\begin{theorem}\label{thm:disj:cd}
The problem
$\stb_{c,d}(A,B)$ 
is
hard for the class $\PiP{2}$.
\end{theorem}
\begin{proof}
According to our notational convention, we have to show that 
$\stb_{c,d}(A,B)$ is $\PiP{2}$-hard, for every finite
$A,B\subseteq \at$.

Let $\forall Y\exists X \varphi$ be a QBF, where $\varphi$ is a CNF
formula over $X\cup Y$. Without loss of generality we can assume that
every clause in $\varphi$ contains at least one literal $x$ or $\neg x$,
for some $x\in X$. Furthermore, we can also assume that $A\cap (X\cup Y)=
\emptyset$ and $B\cap (X\cup Y)=\emptyset$ (if not, variables in $\vph$
can be renamed). We select the primed (fresh) variables so that $A\cap (X'\cup
Y')=\emptyset$ and $B\cap (X'\cup Y')=\emptyset$, as well.

We will construct programs $P(\varphi)$ and $Q(\varphi)$
so that $\forall Y\exists X\varphi$ is true if and only if
$P(\vph)$ and $Q(\vph)$ are stable-equivalent relative to
$\caHB(\ac,B)$. Since the problem to decide whether a given
QBF $\forall Y\exists X \varphi$ is true is $\PiP{2}$-complete, the
assertion will follow.

To construct $P(\vph)$ and $Q(\vph)$ we select an additional atom
$a\notin X\cup X'\cup Y\cup Y'\cup A \cup B$, 
and use $\hat{c}$, as defined in some of the arguments earlier in the paper. We set
\begin{eqnarray*}
R(\varphi) & = &
\{ a\lar x,x';\; x\lar a;\; x'\lar a
\mid x \in X
\} \cup \\
&& \{ y \vee y';\; \lar y,y' \mid y \in Y\} \cup \\
&& \{ a \lar  \hat{c} \mid c\mbox{\ is a clause in\ } \varphi\} \cup \\
&& \{\lar\nt a\}
\end{eqnarray*}
and define
\begin{eqnarray*}
P(\varphi) & = &
\{ x \vee x' \mid   x \in X
\} \cup R(\varphi) \\
Q(\varphi) & = &
\{ x \vee x' \lar u \mid x \in X, u\in \{a\}\cup X\cup X'
\} \cup R(\varphi)
\end{eqnarray*}
To simplify notation, from now on we write $P$ for $P(\vph)$ and $Q$
or $Q(\vph)$.

We note that $\at(P)=\at(Q)$, $\at(P)\subseteq \ac$, $\at(Q)\subseteq\ac$,
$\at(P)\cap B=\emptyset$, and $\at(Q)\cap B=\emptyset$. Thus, to determine
whether $P$ and $Q$ are stable-equivalent with respect to $\caHB(\ac,B)$,
we will focus only on pairs $(N,M)\in\SE_\ac^B(P)$ and $(N,M)\in\SE_\ac^B(Q)$
that satisfy $N\subseteq M\subseteq\at(P)$ (cf. Lemma \ref{lemma:ue:a}). By
Lemma \ref{lemma:ue:b}, to identify such pairs, we need to consider models
(contained in $\at(P)=\at(Q)$) of the two programs, and models (again 
contained in $\at(P)=\at(Q)$) of the reducts of the two programs with respect
to their models. From now on in the proof, whenever we use the term ``model''
(of a program or the reduct of a program) we assume that it is a subset of 
$\at(P)=\at(Q)$.

First, one can check that the models of $P$ and $Q$ coincide and are of the form:\\
\hspace*{0.35in}\parbox{5in}{
\begin{enumerate}
\item
$I \cup (Y\setminus I)' \cup X\cup X'\cup \{a\}$, for each $I\subseteq Y$.
\end{enumerate}
}

Next, we look at models of the reducts of $P$ and $Q$ with respect to 
their models, that is, sets of the form (1). Let $M$ be such a set. Since
$a\in M$, then every model of $P$ is a model of $P^M$, and the same holds 
for $Q$.

However, $P^M$ and $Q^M$ have additional models. First, each reduct has as
its models sets of the form\\
\hspace*{0.35in}\parbox{5in}{
\begin{enumerate}
\item[2.]
$I\cup (Y\setminus I)'\cup J \cup (X\setminus J)'$,
where $J\subseteq X$, $I\subseteq Y$ and $I\cup J\models \varphi$.
\end{enumerate}
}
Furthermore, $Q^M$ has additional models, namely, sets of the form\\
\hspace*{0.35in}\parbox{5in}{
\begin{enumerate}
\item[3.]
$I\cup (Y\setminus I)'$, for each $I\subseteq Y$.
\end{enumerate}
}
Indeed, it is easy to check that $I\cup (Y\setminus I)'$ satisfies all rules
of $Q^M$ (in the case of the rules $a \leftarrow \hat{c}$, we use the fact
that every sequence $\hat{c}$ contains an atom $x$ or $x'$ for some $x\in X$).

We will now show that $\forall Y\exists X \varphi$ is true if and only if
$P$ and $Q$ are stable-equivalent relative to $\caHB(\ac,B)$. To this end,
we will show that $\forall Y\exists X \varphi$ is true 
if and only if $\SE_\ac^{B}(P)=\SE_\ac^{B}(Q)$.

We recall that since $\at(P)=\at(Q)\subseteq\ac$ and
$\at(P)\cap B = \at(Q) \cap B = \emptyset$, we can use Lemmas \ref{lemma:ue:a} and \ref{lemma:ue:b}.
%
Thus, if $M\subseteq\at(P)$, $(N,M)\in\SE_\ac^{B}(P)$ if and only if
$M$ is a set of type (1), that is, 
$M=I\cup (Y\setminus I)'\cup X\cup X'\cup \{a\}$,
for some $I\subseteq Y$, and
either $N=M$ or $N$ is a set of type (2), 
that is, $N=I\cup (Y\setminus I)'\cup J
\cup (X\setminus J)'$, for some $J\subseteq X$ such that $I\cup J\models
\varphi$.

The same pairs $(N,M)$ belong to $\SE_\ac^{B}(Q)$ (still under the assumption
that $M\subseteq\at(P)=\at(Q)$). However, $\SE_\ac^{B}(Q)$ contains also
pairs $(N,M)$ where $M$ is a set of type (1), $N=I\cup (Y\setminus I)'$ 
\emph{and} for every $J\subseteq X$, $I\cup J \not\models \vph$ (given that 
the only models of $Q^M$ that are proper \emph{supersets} of $N$ and proper \emph{subsets}
of $M$ are models of type (2), that is precisely what is needed to ensure 
that for every $Z$, $N\subset Z\subset M$ implies $Z\not\models Q^M$).

Let us assume that $\forall Y\exists X \varphi$ is false. Then, there
exists $I\subseteq Y$ such that for every $J\subseteq X$, $I\cup J\not
\models \varphi$. Let $N=I \cup(Y\setminus I)'$ and 
$M=I \cup(Y\setminus I)' \cup X\cup X' \cup \{a\}$.
From our discussion, it is clear that $(N,M)\in \SE_\ac^{B}(Q)$ but 
$(N,M) \notin\SE_\ac^{B}(P)$. Thus, $\SE_\ac^{B}(P)\not=\SE_\ac^{B}(Q)$.

Conversely, if $\forall Y\exists X \varphi$ is true, then for every $I
\subseteq Y$ there is $J\subseteq X$ such that $I\cup J\models\vph$. This
implies that there are no pairs $(N,M)\in \SE_\ac^{B}(Q)$ of the last kind.
Thus, in that case, if $M\subseteq\at(P)$=$\at(Q)$, then $(N,M)\in
\SE_\ac^{B}(P)$ if and only if $(N,M)\in\SE_\ac^{B}(Q)$. 
By Lemma \ref{lemma:ue:a}, $\SE_\ac^{B}(P)= \SE_\ac^{B}(Q)$. 
\end{proof}

\smallskip
Combining Theorem \ref{thm:disj:cd} with Theorem \ref{thm:memb3} yields
the following result. 

\begin{corollary}\label{cor:stbcd}
The problems 
$\stb_{c,d}(A,B)$,
$\stb_{c,d}(\cdot,B)$,
$\stb_{c,d}(A,\cdot)$ and
$\stb_{c,d}(\cdot,\cdot)$,
are $\PiP{2}$-complete. 
\end{corollary}

Next, we consider the problems $\stb_{d,c}(A,B)$,
and
$\stb_{d,d}(A,B)$. 
We have the following simple result.

\begin{lemma}
\label{lemma:dc1}
Let $P$ and $Q$ be programs and $A,B$ subsets of $\at$ such that $\at(P\cup Q)
\cap A=\emptyset$. Then, $P$ and $Q$ are stable-equivalent with respect to
$\caHB(A,B)$ if and only if $P$ and $Q$ have the same stable models.
\end{lemma}
\begin{proof}
Let $R\in \caHB(A,B)$. Since $\at(P\cup Q) \cap A=\emptyset$, we can apply
the splitting theorem~\cite{litu94} to $P\cup R$. It follows that $M$ is a stable model of 
$P\cup R$ if and only if $M=M'\cup M''$, where $M'$ is a stable model of 
$P$ and $M''$ is a stable model of $M''\cup R$. Similarly, $M$ is a stable 
model of $Q\cup R$ if and only if $M=M'\cup M''$, where $M'$ is a stable 
model of $Q$ and $M''$ is a stable model of $M''\cup R$. Thus, the assertion
follows.
\end{proof}

\smallskip
We now use this result to determine the lower bounds on the complexity of
problems $\stb_{d,c}(A,B)$ and $\stb_{d,d}(A,B)$.

\begin{theorem}\label{thm:disj:dc}
The problems $\stb_{d,c}(A,B)$
and $\stb_{d,d}(A,B)$
are hard for the class $\PiP{2}$.
\end{theorem}
\begin{proof}
To be precise, we have to show that 
$\stb_{d,c}(A,B)$
and $\stb_{d,d}(A,B)$ are $\PiP{2}$-hard, for every finite $A,B\subseteq\at$.

It is well known that the problem to decide whether a logic program $P$ has
a stable model is $\Sigma_2^P$-complete \cite{eite-gott-95}. We will reduce this problem to
the complement of $\stb_{d,c}(A,B)$ ($\stb_{d,d}(A,B)$, respectively). 
That will complete the proof.

Thus, let $P$ be a logic program. Without loss of generality, we can assume
that $\at(P)\cap A=\emptyset$ (if not, we can rename atoms in $P$, without 
affecting the existence of stable models). 
Let $f$ be
an atom not in $A$.
and define $Q=\{f,\ \leftarrow f\}$. Clearly, $\at(P\cup Q)\cap A=\emptyset$.
Moreover, $P$ and $Q$ do not have the same stable models if and only if $P$
has stable models. By Lemma \ref{lemma:dc1}, $P$ has stable models if and only 
if $P$ and $Q$ are not stable-equivalent relative to $\caHB(A,\bc)$. Similarly 
(as $B$ is immaterial for the stable-equivalence in that case), $P$ has 
stable models if and only if $P$ and $Q$ are not stable-equivalent relative to 
$\caHB(A,B)$. 
\end{proof}

\smallskip
We now explicitly list all cases, where we are able to give completeness results 
(membership results are from Theorem~\ref{thm:memb3}).

\begin{corollary}\label{cor:stbdd}
The problems 
$\stb_{d,d}(A,B)$,
$\stb_{d,d}(\cdot,B)$,
$\stb_{d,d}(A,\cdot)$ and
$\stb_{d,d}(\cdot,\cdot)$,
are $\PiP{2}$-complete. 
\end{corollary}

\begin{corollary}\label{cor:stbdc}
The problems 
$\stb_{d,c}(A,B)$,
$\stb_{d,c}(\cdot,B)$,
$\stb_{d,c}(A,\cdot)$ and
$\stb_{d,c}(\cdot,\cdot)$,
are $\PiP{2}$-complete. 
\end{corollary}

Finally, we show $\Pi_2^P$-hardness of problems $\stb_{c,c}(A,\cdot)$
and $\stb_{c,c}(\cdot,B)$.

\begin{theorem}\label{thm:disj:cc}
The problems
$\stb_{c,c}(A,\cdot)$
and
$\stb_{c,c}(\cdot,B)$.
are
$\PiP{2}$-hard.
\end{theorem}
\begin{proof}
We first show that the problem $\stb_{c,c}(A,\cdot)$ is $\PiP{2}$-hard,
for every finite $A\subseteq \at$. Let $\forall Y\exists X \varphi$ be 
a QBF, where $\varphi$ is a CNF formula over $X\cup Y$. As in the proof
of Theorem \ref{thm:disj:cd}, without loss of generality we can assume
that every clause in $\varphi$ contains a literal $x$ or $\neg x$,
for some $x\in X$, and that $A\cap (X\cup Y)= \emptyset$ (if not, variables
in $\varphi$ can be renamed).
  
Let $P(\varphi)$ and $Q(\vph)$ be the programs used in the proof of 
Theorem~\ref{thm:disj:cd}, where we choose primed variables so that 
$A\cap (X'\cup Y')=\emptyset$. We define $B=\at(P)$. We have that 
$\at(P)\subseteq \ac$ and $\at(P)\cap \bc=\emptyset$.

We recall that the argument used in the proof of Theorem
\ref{thm:disj:cd} to show that $\forall Y\exists X \varphi$ is true if and
only if $P(\vph)$ and $Q(\vph)$ are stable-equivalent with respect to
$\caHB(\ac,B)$ does not depend on the finiteness of $B$ but only on the
fact that $B\cap\at(P)=\emptyset$. Thus, the same argument shows that
$\forall Y\exists X \varphi$ is true if and only if $P(\vph)$ and $Q(\vph)$
are stable-equivalent with respect to $\caHB(\ac,\bc)$. It follows that
$\stb_{c,c}(A,\cdot)$ is $\Pi_2^P$-hard.

Next, we show that the problem $\stb_{c,c}(\cdot,B)$ is $\PiP{2}$-hard,
for every finite $B\subseteq \at$. We reason as in the proof of 
Theorem~\ref{thm:disj:dc}. That is, we construct a reduction from the
problem to decide whether a logic program has no stable models. Specifically,
let $P$ be a logic program. We define $A=\at(P)$. Clearly, we have $\at(P)
\cap \ac =\emptyset$. We recall the argument used in Theorem~\ref{thm:disj:dc} to
show that $P$ has stable models if and only if $P$ and $Q=\{f,\;\leftarrow
f\}$ are not stable-equivalent with respect to $\caHB(A,B)$ does not depend on 
the finiteness of $A$ nor on $B$. Thus, it follows that $P$ has stable models 
if and only if $P$ and $Q=\{f,\;\leftarrow f\}$ are not stable-equivalent with
respect to $\caHB(\ac,\bc)$ and the $\Pi_2^P$-hardness of $\stb_{c,c}(\cdot,B)$
follows. 
\end{proof}

\smallskip
We put the things together using 
Theorem~\ref{thm:norm}  for the coNP-hardness and
Theorem~\ref{thm:disj:cc} for the $\PiP{2}$-hardness.
The matching upper bounds are from Theorem~\ref{thm:memb3}.

\begin{corollary}\label{cor:stbcc}
The problem $\stb_{c,c}(A,B)$ is coNP-complete. 
The problems
$\stb_{c,c}(\cdot,B)$,
$\stb_{c,c}(A,\cdot)$ and
$\stb_{c,c}(\cdot,\cdot)$,
are $\PiP{2}$-complete. 
\end{corollary}

\section{Discussion}

We studied the complexity of deciding relativized 
hyperequivalence of programs under the semantics of stable, 
supported and supported minimal models. We focused on problems
$\sem_{\delta,\epsilon}(\alpha,\beta)$, where at least one of 
$\delta$ and $\epsilon$ equals $c$, that is, at least one of the 
alphabets for the context problems is determined as the complement of 
the corresponding set $A$ or $B$. As we noted, such problems arise 
naturally in the context of modular design of logic programs, yet they
have received essentially no attention so far.

\begin{table}
\caption{Complexity of $\sem_{\delta,\varepsilon}(\alpha,\beta)$; all entries are completeness results.}\label{tab}
\label{tab1}
\begin{minipage}{\textwidth}
\begin{tabular}{cccccccc}
\hline
\hline
$\delta$ & $\varepsilon$ & $\alpha$ & $\beta$  &
$\spp$ &
$\sppm$ & 
$\stb$ & 
$\stb^n$ \\
\hline 
$d$ & $d$ &&& coNP & $\PiP{2}$ & $\PiP{2}$ & coNP \\
\noalign{\vspace {.5cm}}
$d$ & $c$ & 
& $\cdot$ & coNP & $\PiP{2}$ & $\PiP{2}$ & coNP \\
$d$ & $c$ & 
& B & coNP & coNP &  $\PiP{2}$ & coNP \\
\noalign{\vspace {.5cm}}
$c$ & $c$ & $\cdot$ or $A\neq\emptyset$ & $\cdot$ & coNP & $\PiP{2}$ & $\PiP{2}$ & coNP \\
$c$&  $c$ & $\emptyset$ & $\cdot$ & coNP & coNP & $\PiP{2}$ & coNP \\
$c$ & $c$ & $\cdot$ & $B$ & coNP & coNP & $\PiP{2}$ & coNP \\
$c$ & $c$ & $A$ & $B$ & coNP & coNP & coNP & coNP \\
\noalign{\vspace {.5cm}}
$c$&  $d$ & $\cdot$ or $A\neq \emptyset$  &  & coNP & $\PiP{2}$ & $\PiP{2}$ & coNP \\
$c$&  $d$ & $\emptyset$ & & coNP & coNP & $\PiP{2}$ & coNP \\
\hline
\hline
\end{tabular}
\end{minipage}
\end{table}

Table \ref{tab1} summarizes the results (for the sake of completeness
we also include the complexity of direct-direct problems). It shows that the
problems concerning supp-equivalence (no normality restriction), and
stable-equivalence for normal programs are all coNP-complete 
(cf.\ Corollaries~\ref{cor:supp} and~\ref{cor74}, respectively).
The situation is more diversified for suppmin-equivalence and
stable-equivalence (no normality restriction) with some problems being
coNP- and others $\PiP{2}$-complete. For suppmin-equivalence lower
complexity requires that $B$ be a part of problem specification, or
that $A$ be a part of problem specification and be set to $\emptyset$.
The results for direct-direct problems were known earlier \cite{tw08a}, 
the results for the direct-complement problems are by Corollary~\ref{cor:suppmindc},
for the complement-complement problems results are by Corollary~\ref{cor:suppmincc}, and
for the complement-direct problems results are by Corollary~\ref{cor:suppmincd}.
For stable-equivalence, the lower complexity only holds for the
complement-complement problem with both $A$ and $B$ fixed as part of
the problem specification.
The results for direct-direct (direct-complement, complement-complement, 
complement-direct, respectively) problems are by Corollary~\ref{cor:stbdd} 
(\ref{cor:stbdc}, \ref{cor:stbcc}, \ref{cor:stbcd}, respectively) in this
paper. 
We also note that the complexity of problems
for stable-equivalence is always at least that for suppmin-equivalence.

Our research opens questions worthy of further investigations. First,
we believe that results presented here may turn out important for building
``intelligent'' programming environments supporting development of logic 
programs. For instance, a programmer might want to know the effect of 
changes she just made to a program (perhaps already developed earlier) 
that represents a module of a larger project. One way to formalize that 
effect is to define it as the maximal class of contexts of the form 
$\caHB(A',B')$ with respect to which the original and the revised versions 
of the program are equivalent (say under the stable-model semantics). 
The sets $A'$ and $B'$ appearing in the specification of such a class of 
contexts will be of the form $A^c$ and $B^c$, for some finite sets $A$
and $B$. Finding the appropriate sets $A$ and $B$ would provide useful 
information to the programmer. Our results on the complexity of the 
complement-complement version of the hyperequivalence problem and their
proofs may yield insights into the complexity of finding such sets $A$
and $B$, and suggest algorithms.

Second, there are other versions of hyperequivalence that need to be
investigated. For instance, while
stable-equivalence when only parts of models are compared (projections
on a prespecified set of atoms) was studied
\cite{EiterTW05,Oetsch07}, no similar results are available for supp- and 
suppmin-equivalence. Also the complexity of the corresponding 
complement-direct, direct-complement and complement-com\-ple\-ment problems 
for the three semantics in that setting has yet to be established.

\section*{Acknowledgments}
This work was partially
supported by the NSF grant IIS-0325063, the KSEF grant KSEF-1036-RDE-008,
and by the Austrian Science Fund~(FWF) under grants P18019 and P20704.

\section*{Appendix}

We present here proofs of some technical results we needed in the paper.
We first prove Lemma \ref{lem:new}. We start with two auxiliary results.

\begin{lemma}
\label{in1}
Let $P$ be a program and $A,B\subseteq\at$. Let $y\in X$ be such
that $y\notin \At(P)\cup A$. Then $(X,Y)
\in\Mod_\ac^B(P)$ if and only if $(\Xy,\Yy)\in\Mod_\ac^B(P)$.
\end{lemma}
\begin{proof}
($\Rightarrow$) Since $Y\in\Mod_\ac(P)$, $Y\models P$ and $Y
\setminus T_P(Y)\subseteq \ac$. We have $y\notin \At(P)$. Thus, $\Yy
\models P$ and $T_P(Y)=T_P(\Yy)$. Since $\Yy\subseteq Y$, $(\Yy)\setminus
T_P(\Yy)\subseteq\ac$. It follows that $\Yy\in\Mod_\ac(P)$. Thus, the
condition (1) for $(\Xy,\Yy)\in\Mod_\ac^B(P)$ holds. The condition (2) for 
$(\Xy,\Yy)\in\Mod_\ac^B(P)$ is evident.

Let $Z\subset\Yy$ be such that $Z|_{\ac\cup B}=(\Yy)|_{\ac\cup B}$. Let
$Z'=Z\cup\{y\}$. We have $y\in X$ and so, $y\in Y$. Hence, $Z'\subset Y$.
Since $y\notin A$, $y\in\ac$.
Thus, $Z'|_{\ac\cup B}=Y|_{\ac\cup B}$. It follows that $Z'\not\models
P$ and, consequently, $Z\not\models P$ (as $y\notin \at(P)$).
Thus, the condition (3) for $(\Xy,\Yy)\in\Mod_\ac^B(P)$ holds. 

Next, let $Z\subset\Yy$ be such that $Z|_{B}=(\Xy)|_{B}$ and 
$Z|_{\ac}\supseteq (\Xy)|_{\ac}$. As before, let $Z'=Z\cup\{y\}$. 
Since $y\in X$ and $y\in Y$ (see above), $Z'\subset Y$, $Z'|_{B}=X|_{B}$
and $Z'|_{\ac}
\supseteq X|_{\ac}$. Thus, $Z'\not\models P$. Since $y\notin\At(P)$,
$Z\not\models P$ and the condition (4) for $(\Xy,\Yy)\in\Mod_\ac^B(P)$
holds. 

Finally, let $(\Xy)|_B=(\Yy)|_B$. Clearly, it follows that $X|_B=
Y|_B$. Thus, $Y\setminus T_P(Y)\subseteq X$. Since $y\notin\at(P)$, 
$T_P(Y)= T_P(\Yy)$. It follows that $(\Yy)\setminus T_P(\Yy)\subseteq
\Xy$. Consequently, the condition (5) for $(\Xy,$ $\Yy)\in\Mod_\ac^B(P)$
is satisfied, as well.

\smallskip
\noindent
($\Leftarrow$) By the assumption, we have $(\Xy,\Yy)\in\Mod^B_\ac(P)$. Thus,
$\Yy\in\Mod_\ac(P)$ and, consequently, $\Yy$ is a model of $P$.
Since $y\notin\at(P)$, $Y$ is a model of $P$. We also have $(\Yy)
\setminus T_P(\Yy)\subseteq\ac$. Since $y\notin\at(P)$, $T_P(\Yy)
=T_P(Y)$. Thus, as $y\in\ac$, $Y\setminus T_P(Y)\subseteq\ac$. That
is, the condition (1) for $(X,Y)\in\Mod_\ac^B(P)$ holds. The 
condition (2) follows from $y\in\ac$ and $\Xy \subseteq 
(\Yy)|_{\ac\cup B}$. 

Let $Z\subset Y$ be such that $Z|_{\ac\cup B}=Y|_{\ac\cup B}$. It
follows that $y\in Z$ (we recall that $y\in X\subseteq Y$ and $y\in\ac$). Let 
$Z'=Z\setminus\{y\}$. We have $Z'\subset\Yy$ and $Z'|_{\ac\cup B}=
(\Yy)|_{\ac\cup B}$. Thus, $Z'\not\models P$ and, consequently, $Z\not
\models P$. It follows that the condition (3) for $(X,Y)\in
\Mod_\ac^B(P)$ holds.

Let $Z\subset Y$ be such that $Z|_{B}=X|_B$ and $Z|_{\ac}\supseteq 
X|_{\ac}$. Since $y\in X$ and $y\in\ac$, $y\in Z$. Let $Z'=Z\setminus
\{y\}$. It follows that $Z'\subset\Yy$, $Z'|_B=(\Xy)|_B$, and 
$Z'|_\ac\supseteq(\Xy)|_\ac$. Hence, $Z'\not\models P$ and so, $Z\not
\models P$. In other words, the condition (4) for $(X,Y)\in
\Mod_\ac^B(P)$, holds. 

Finally, let $X|_B=Y|_B$. Clearly, $(\Xy)|_B=(\Yy)|_B$ and so, 
$(\Yy)\setminus T_P(\Yy)\subseteq \Xy$. Since $T_P(\Yy)=T_P(Y)$,
we obtain $Y\setminus T_P(Y)\subseteq X$. Thus, (5) for $(X,Y)\in
\Mod_\ac^B(P)$, holds.  
\end{proof}

\begin{lemma}
\label{in3}
Let $P$ be a program, $A,B\subseteq\at$.
If $X|_{B}\subset Y|_{B}$, $y\in(Y\setminus X)\setminus(\At(P)\cup
A)$, and $(\Yy)|_{B}\not=X|_{B}$, then $(X,Y)\in\Mod_\ac^{B}
(P)$ if and only if $(X,\Yy)\in\Mod_\ac^{B}(P)$.
\end{lemma}
\begin{proof}
($\Rightarrow$) The arguments for the conditions (1), (2) and 
(3) for $(X,\Yy)\in\Mod_\ac^B(P)$ are essentially the same as in 
Lemma \ref{in1} (although the argument for the condition (2) requires also
the assumption that $y\notin X$).

Next, let $Z\subset\Yy$ be such that $Z|_{B}=X|_{B}$ and $Z|_{\ac}
\supseteq X|_{\ac}$. Then $Z\subset Y$ and so, $Z\not\models P$. Thus,
the condition (4) for $(X,\Yy)\in\Mod_\ac^B(P)$ holds.

Finally, $(\Yy)|_B \not=X|_B$; the condition (5) for 
$(X,\Yy)\in \Mod_\ac^B(P)$ is thus trivially true.

\smallskip
\noindent
($\Leftarrow$) As above, the arguments for the conditions (1), (2) and
(3) for $(X,Y)\in\Mod_\ac^B(P)$ are the same as in Lemma \ref{in1}.
%

Let $Z\subset Y$ be such that $Z|_{B}=X|_{B}$ and $Z|_{\ac}\supseteq
X|_{\ac}$. Since $(\Yy)|_{B}\not=X|_{B}$, $Z\not=\Yy$. Thus, $Z
\subset\Yy$ and so, $Z\not\models P$. That is, the condition (4) for 
$(X,Y)\in\Mod_\ac^B(P)$, holds.
Finally, since $X|_B\subset Y|_B$, the condition (5) for 
$(X,Y)\in\Mod_\ac^B(P)$, holds, as well. 
\end{proof}

\smallskip
We are now ready to prove Lemma \ref{lem:new}.

\medskip
\noindent
\textit{Lemma \ref{lem:new}}\\
\textit{
Let $P,Q$ be programs and $A,B\subseteq\at$.
\begin{enumerate}
\item 
If $(X,Y)\in\Mod_\ac^B(P)\setminus\Mod_\ac^B(Q)$ then
there is $(X',Y')\in\Mod_\ac^B(P)\setminus\Mod_\ac^B(Q)$ such that 
$Y'\subseteq\At(P\cup Q)\cup A$.
\item 
If $(X,Y)\in\Mod_\ac^B(P)$ and $T_P(Y)|_B\not=T_Q(Y)|_B$,
then there is $(X',Y')\in\Mod_\ac^B(P)$ such that $T_P(Y')|_B\not=
T_Q(Y')|_B$ and $Y'\subseteq\At(P\cup Q)\cup A$.
\end{enumerate}
}
\begin{proof}
(1) Let $(X,Y)\in\Mod_\ac^B(P)\setminus\Mod_\ac^B(Q)$ and 
let $y\in X$ be such that $y\notin \At(P\cup Q)\cup A$. Then, by
Lemma \ref{in1}, $(\Xy,\Yy)\in\Mod_\ac^B(P)\setminus\Mod_\ac^B(Q)$.
By repeating this process, we arrive at a pair $(X'',Y'')\in
\Mod_\ac^B(P)\setminus\Mod_\ac^B(Q)$ such that $X''\subseteq
\At(P\cup Q)\cup A$.

If $X''|_B=Y''|_B$, then $Y''\setminus T_P(Y'')\subseteq X''$. Thus,
$Y''\subseteq T_P(Y'')\cup X''\subseteq\At(P\cup Q)\cup A$.
Thus, let us consider the other possibility that $X''|_B\subset 
Y''|_B$ (indeed, as $X''\subseteq Y''|_{\ac\cup B}\subseteq Y''$, 
there are no other possibilities). Let $y\in(Y''\setminus X'')\setminus
(\At(P\cup Q)\cup A)$ be such that $(Y''\setminus\{y\})|_B\not=
X''|_B$. By Lemma \ref{in3}, $(X'',Y''\setminus\{y\})\in
\Mod_\ac^B(P)\setminus\Mod_\ac^B(Q)$. By repeating this process, we
arrive at a pair $(X',Y')\in\Mod_\ac^B(P)\setminus\Mod_\ac^B(Q)$ 
such that for every $y\in(Y'\setminus X')\setminus(\At(P\cup Q)
\cup A)$, $(Y'\setminus\{y\})|_B=X'|_B$. Since $X'=X''$, $X'\subseteq
\At(P\cup Q)\cup A$.

We also note that for every $y\notin X'$, $(Y'\setminus\{y\})\supseteq 
X'$ (as $Y'\supseteq X'$) and so, $(Y'\setminus\{y\})
|_\ac\supseteq X'|_\ac$.  
We will now show that $Y'\subseteq\At(P\cup Q)\cup A$. To this end,
let us assume that there is $y\in Y'$ such that $y\notin\At(P\cup Q)\cup
A$. Since $X'\subseteq\At(P\cup Q)\cup A$, $y\notin X'$.
Thus, $y\in (Y'\setminus X')\setminus(\At(P\cup Q)\cup A)$. It 
follows that $(Y'\setminus\{y\})|_B=X'|_B$ and $(Y'\setminus\{y
\})|_\ac\supseteq X'|_\ac$. Since $Y'\setminus\{y\}\subset Y'$ and
$(X',Y')\in\Mod_\ac^B(P)$, $Y'\setminus\{y\}\not\models P$. On the
other hand, $Y'\models P$ and, since $y\notin\At(P)$, $Y'\setminus\{y\}
\models P$, a contradiction. 

\smallskip
\noindent
(2) It is easy to see that if we apply the construction described in (1)
to $(X,Y)$ we obtain $(X',Y')$ such that $Y'\subseteq\at(P\cup Q)\cup 
A$ and $T_P(Y')|_B\not=T_Q(Y')|_B$. Indeed, in every step of the 
construction, we eliminate an element $y$ such that $y\notin\At(P\cup 
Q)$, which has no effect on the values of $T_P$ and $T_Q$.
\end{proof}

\medskip
\noindent
\textit{Lemma \ref{cor:usuppmin:2}}\\
\textit{
Let $P,Q$ be programs and $B\subseteq \at$.
If $\Mod_\at^B(P)\not=\Mod_\at^B(Q)$, then there is
$Y\subseteq\at(P\cup Q)$
such that $Y$ is a model of exactly one of $P$ and $Q$, or there is
$a\in Y$ such that $(Y\setminus\{a\},Y)$
belongs to exactly one of $\Mod_\at^B(P)$ and $\Mod_\at^B(Q)$.
}

\smallskip
\noindent
\begin{proof} Let us assume that $P$ and $Q$ have the same models
(otherwise, there is $Y\subseteq\At(P\cup Q)$ that is a model of 
exactly one of $P$ and $Q$, and the assertion follows). 
Without loss of generality we can assume that there is 
$(X,Y)\in\Mod_\at^B(P)\setminus\Mod_\at^B(Q)$.
Moreover, by Lemma \ref{lem:new}, we can assume that 
$Y\subseteq\at(P\cup Q)$ (recall $\at^c=\emptyset$).
It follows that $(X,Y)$ satisfies the conditions (1)-(5) for $(X,Y)
\in\Mod_\at^B(P)$. Since $P$ and $Q$ have the same models, $(X,Y)$
satisfies the conditions (1)-(4) for $(X,Y)\in\Mod_\at^B(Q)$. Hence,
$(X,Y)$ violates the condition (5) for $(X,Y)\in\Mod_\at^B(Q)$, that
is, $X|_B=Y|_B$ and $Y\setminus T_Q(Y)\not\subseteq X$ hold.
In particular, there is $a\in (Y\setminus T_Q(Y))\setminus X$. We will 
show that $(Y\setminus\{a\},Y)\in\Mod_\at^B(P)$ and $(Y\setminus\{a\},Y)
\notin\Mod_\at^B(Q)$.

Since $(X,Y)\in \Mod_\at^B(P)$, $Y$ is a model of $P$ and so, 
$Y\in\Mod_\at(P)$. Next, obviously, $Y\setminus \{a\}\subseteq Y$.
Thus, the conditions (1) and (2) for $(Y\setminus\{a\},Y)\in\Mod_\at^B(P)$
hold. The condition (3) is trivially true.

Further, let $Z\subset Y$ be such that $Z|_B=(Y\setminus\{a\})|_B$ and
$Z\supseteq Y\setminus \{a\}$. 
Then $Z=Y\setminus \{a\}$. We have $Y|_B=X|_B$, $a\in Y$, and $a\notin
X$. Thus, $a\notin B$. It follows that $(Y\setminus\{a\})|_B=X|_B$ 
and $X\subseteq Y\setminus \{a\}$. Since $Y\setminus\{a\}\subset Y$ and
$(X,Y)\in\Mod_\at^B(P)$, $Y\setminus\{a\}\not\models P$,
that is, $Z\not\models P$. Thus, the condition (4) for
$(Y\setminus\{a\},Y)\in\Mod_\at^B(P)$ holds.

Since $a\notin B$, $(Y\setminus\{a\})|_B=Y|_B$. Thus, we also have to 
verify the condition (5). We have $Y\setminus T_P(Y)\subseteq X$ (we
recall that $Y|_B=X|_B$) and so, $a\notin Y\setminus T_P(Y)$. Consequently,
$Y\setminus T_P(Y)\subseteq Y\setminus\{a\}$. Hence, the condition (5)
holds and $(Y\setminus\{a\},Y)\in\Mod_\at^B(P)$.
On the other hand, $a\in Y\setminus T_Q(Y)$ and $a\notin Y\setminus
\{a\}$. Thus, the condition (5) for $(Y\setminus\{a\},Y)\in\Mod_\at^B(Q)$
does not hold and so, $(Y\setminus\{a\},Y)\notin\Mod_\at^B(Q)$.
\end{proof}

\smallskip
Next, we present proofs of the technical results needed in Section
\ref{stabeq}: Lemmas \ref{stin1e} and \ref{lemma:ue}. First, we establish
some auxiliary results. We start with conditions providing conditions 
restricting $X$ and $Y$ given that $(X,Y)\in \SE_{A}^{B}(P)$.

\begin{lemma}
\label{stin0}
Let $P$ be a program and $A,B\subseteq\at$.
If $(X,Y)\in \SE_A^B(P)$
then $X\subseteq Y\subseteq\at(P)\cup A$.
\end{lemma}
\begin{proof}
Let $(X,Y)\in \SE_A^B(P)$. The inclusion $X\subseteq
Y$ follows from the condition (2). To prove $Y\subseteq\at(P)\cup A$,
let us assume to the contrary that $Y\setminus(\at(P)\cup A)\not=\emptyset$.
Let $y\in Y\setminus (\at(P)\cup A)$.
We have $Y\models P$ and thus $Y\models P^Y$.  Since
$y\notin \at(P)$, $y\notin \at(P^Y)$. Thus, $Y\setminus \{y\} \models  P^Y$.
Since $y\notin A$, taking $Z=\Yy$ shows that $(X,Y)$ violates the
condition (3) for $(X,Y)\in \SE_A^B(P)$, a contradiction.
\end{proof}

\smallskip
The next two lemmas show that some atoms are immaterial for the
membership of a pair $(X,Y)$ in $\SE_A^B(P)$.

\begin{lemma}
\label{stin1a}
Let $P$ be a program, $A,B,X,Y\subseteq\at$, $y\in (X\cap Y)\setminus\at(P)$,
and $y\in A$. Then $(X,Y)\in\SE_A^B(P)$  if and only if $(\Xy,\Yy) \in
\SE_A^B(P)$.
\end{lemma}
\begin{proof}
We will show that each of 
the conditions (1) - (5) for $(X,Y)\in\SE_A^B(P)$ is equivalent to its
counterpart for $(\Xy,\Yy) \in\SE_A^B(P)$.

The case of the condition (1) is clear. Since $y\notin\at(P)$, 
$Y\models P$ if and only if $\Yy\models P$. It is also evident that
$X=Y$ if and only if $\Xy=\Yy$, $X\subseteq Y|_{A\cup B}$ if and only
if $\Xy\subseteq (\Yy)|_{A\cup B}$, and $X|_A\subset Y|_A$ if and only
if $(\Xy)|_A\subset (\Yy)|_A$. Thus, the corresponding conditions (2) are also
equivalent.

Let us assume the condition (3) for $(X,Y)\in\SE_A^B(P)$. Let $Z\subset
\Yy$ be such that $Z|_A=(\Yy)|_A$. Let $Z'=Z\cup \{y\}$. Then $Z'\subset Y$
and $Z'|_A=Y|_A$ (as $y\in Y$). By the condition (3) for $(X,Y)\in\SE_A^B(P)$,
$Z'\not\models P^Y$. Since $y\notin\at(P)$, $Z\not\models P^\Yy$, and so,
the condition (3) for $(\Xy,\Yy) \in\SE_A^B(P)$ follows. Conversely,
let us assume the condition (3) for $(\Xy,\Yy)\in\SE_A^B(P)$ and let
$Z\subset Y$ be such that $Z|_A=Y|_A$. It follows that $y\in Z$. We set
$Z'=Z\setminus\{y\}$. Clearly, $Z'\subset \Yy$ and $Z'|_A=(\Yy)|_A$. Thus,
$Z'\not\models P^\Yy$. As $y\notin\at(P)$, $Z\not\models P^Y$ and,
so, the condition (3) for $(X,Y)\in\SE_A^B(P)$ follows.

Next, let us assume the condition (4) for $(X,Y)\in\SE_A^B(P)$. Let
$Z\subset \Yy$ be such that $Z|_B\subset (\Xy)|_B$ and $Z|_A\supseteq (\Xy)|_A$, 
or $Z|_B\subseteq (\Xy)|_B$ and $Z|_A\supset (\Xy)|_A$. Let $Z'=Z\cup\{y\}$.
We have $Z'\subset Y$. Moreover,
it is evident that $Z'|_B\subset X|_B$ and $Z'|_A\supseteq X|_A$,
or $Z'|_B\subseteq X|_B$ and $Z'|_A\supset X|_A$. Thus, $Z'\not\models P^Y$
and so, $Z\not\models P^\Yy$. Similarly, let the condition (4) for 
$(\Xy,\Yy)\in\SE_A^B(P)$ hold. Let $Z\subset Y$ be such that
$Z|_B\subset X|_B$ and $Z|_A\supseteq X|_A$, or $Z|_B\subseteq X|_B$ and 
$Z|_A\supset X|_A$. Since $y\in X$ and $y\in A$, $y\in Z$. We define 
$Z'=Z\setminus \{y\}$ and note that $Z'\subset\Yy$. Moreover, as $y\in X$ and
$y\in Y$, $Z'|_B\subset (\Xy)|_B$ and $Z'|_A\supseteq (\Xy)|_A$, or $Z'|_B
\subseteq (\Xy)|_B$ and $Z'|_A\supset (\Xy)|_A$. Thus, $Z'\not\models P^\Yy$
and so, $Z\not\models P^Y$.
 
Finally, a similar argument works also for the condition (5). Let
the condition (5) for $(X,Y)\in\SE_A^B(P)$ hold. Thus, there is $Z
\subseteq Y$ such that $X|_{A\cup B}=Z|_{A\cup B}$ and $Z\models P^Y$.
Let $Z'=Z\setminus\{y\}$. Since $y\in X$ and $y\in A$, $y\in Z$. Thus,
$Z'\subseteq \Yy$ and $(\Xy)|_{A\cup B}=Z'|_{A\cup B}$. Moreover, since
$Z\models P^Y$, $Z'\models P^\Yy$. Conversely, let the condition (5) for 
$(\Xy,\Yy)\in\SE_A^B(P)$ hold. Then, there is $Z\subseteq\Yy$ such that
$Z|_{A\cup B}=(\Xy)|_{A\cup B}$ and $Z\models P^\Yy$. Let $Z'=Z\cup\{y\}$.
Then $Z'\subseteq Y$, $Z'|_{A\cup B}=X|_{A\cup B}$ and $Z'\models P^Y$. 
\end{proof}

\begin{lemma}
\label{stin1b}
Let $P$ be a program, $A,B,X,Y\subseteq\at$ and
$y\in (Y\setminus (X\cup \at(P)))\cap A$. 
If $|(Y\setminus (X\cup \at(P)))\cap A|>2$, then
$(X,Y)\in \SE_A^B(P)$ if and only 
if $(X,\Yy) \in\SE_A^B(P)$.
\end{lemma}
\begin{proof}
Since $|(Y\setminus(X\cup \at(P)))\cap A|>2$, there are $y',y''\in 
(Y\setminus (X\cup \at(P))) \cap A$ such
that $y,y',y''$ are all distinct. 
As before, we will show that each of
the conditions (1) - (5) for $(X,Y)\in\SE_A^B(P)$ is equivalent to its
counterpart for $(X,\Yy) \in\SE_A^B(P)$.

The case of the condition (1) is evident.
By our assumptions, neither $X=Y$ nor $X=\Yy$. Moreover, 
$X\subseteq Y|_{A\cup B}$ if and only if $X\subseteq (\Yy)|_{A\cup
B}$ and $X|_A\subset Y|_A$ if and only if $X|_A\subset (\Yy)|_A$
(since $y,y'\in Y$ and $y,y'\in A$). Thus, the 
corresponding versions of the condition (2) are also equivalent.
The case of the condition (3) can be argued in the same way as it was 
in Lemma \ref{stin1a}.

Let us assume the condition (4) for $(X,Y)\in\SE_A^B(P)$. Let
$Z\subset \Yy$ be such that 
$Z|_B\subset X|_B$ and $Z|_A\supseteq X|_A$, or 
$Z|_B\subseteq X|_B$ and $Z|_A\supset X|_A$. Clearly, $Z\subset Y$. 
Consequently, by the condition (4) for $(X,Y)\in\SE_A^B(P)$,
$Z\not\models P^Y$ and so, $Z\not\models P^\Yy$. Thus
the condition (4) for $(X,\Yy)\in\SE_A^B(P)$ holds. 

Conversely, let the condition (4) for $(X,\Yy)\in\SE_A^B(P)$ hold.
Let $Z\subset Y$ be such that 
$Z|_B\subset X|_B$ and $Z|_A\supseteq X|_A$, or
$Z|_B\subseteq X|_B$ and $Z|_A\supset X|_A$. 
If $Z\subset \Yy$, then $Z\not\models P^\Yy$ (as the condition (4) for 
$(X,\Yy)\in\SE_A^B(P)$ holds). Thus, $Z\not\models P^Y$.
Otherwise, i.e.\ for $Z=Y\setminus\{y\}$,
we have $y',y''\in Z$. Let $Z'=Z\setminus\{y,y'\}$. It 
follows that $Z'\subset \Yy$ and $Z'|_A\supset X|_A$ (the former, as
$y'\in \Yy\setminus Z'$; the later, as $y''\in Z'|_A\setminus X|_A$).
Thus, $Z'\subset \Yy$, $Z'|_B\subseteq X|_B$ and $Z'|_A\supset 
X|_A$. Consequently, $Z'\not\models P^\Yy$ (again, as the condition 
(4) for $(X,\Yy)\in\SE_A^B(P)$ holds). Thus, also in that case, 
$Z\not\models P^Y$. It follows that the condition (4) for $(X,Y)\in
\SE_A^B(P)$ holds.

Finally, for the condition (5) we reason as follows. Let
the condition (5) for $(X,Y)\in\SE_A^B(P)$ hold. Thus, there is $Z
\subseteq Y$ such that $X|_{A\cup B}=Z|_{A\cup B}$ and $Z\models P^Y$.
Clearly, $y\notin Z$ (as $y\notin X$ and $y\in A$). Thus, $Z\subseteq
\Yy$ and so $Z\models P^Y$ follows.
Conversely, let the condition (5) for $(X,\Yy)\in\SE_A^B(P)$ hold.
Then, there is $Z\subseteq \Yy$ such that $Z|_{A\cup B}=X|_{A\cup
B}$ and $Z\models P^Y$. Clearly, we also have $Z\subseteq Y$ and so,
the condition (5) for $(X,Y)\in\SE_A^B(P)$ follows.
\end{proof}

\smallskip
Finally, we note that the membership of a pair $(X,Y)$, where 
$X\subseteq\at(P)$, in $\SE_\ac^B(P)$ does not depend on specific
elements in $Y\setminus \at(P)$ but only on their number.

\begin{lemma}
\label{stin1c}
Let $P$ be a program, $A,B\subseteq\at$, $X,Y\subseteq\at(P)$, and
$Y',Y''\subseteq A\setminus\at(P)$. If $|Y'|=|Y''|$ then
$(X,Y\cup Y')\in \SE_A^B(P)$ if and only
if $(X,Y\cup Y'') \in\SE_A^B(P)$.
\end{lemma}
\begin{proof}
It is clear that the corresponding conditions (1) - (5) for 
$(X,Y\cup Y')\in \SE_A^B(P)$ and $(X,Y\cup Y'') \in\SE_A^B(P)$,
respectively are equivalent to each other.
\end{proof}

\smallskip
Lemmas \ref{stin0} - \ref{stin1c} allow us to prove Lemma \ref{stin1e}.

\medskip
\noindent
\textit{Lemma \ref{stin1e}}\\
\textit{Let $P,Q$ be programs and $A,B\subseteq\at$. If 
$(X,Y)\in\SE_A^B(P)\setminus
\SE_A^B(Q)$, then there are sets $X',Y'\subseteq\at(P\cup Q)$,
such that at least one of the following conditions holds:
\begin{enumerate}
\item[i.] $(X',Y')\in\SE_A^B(P)\setminus \SE_A^B(Q)$
\item[ii.] 
$A\setminus\at(P\cup Q)\neq\emptyset$ and
for every $y,z\in A\setminus\at(P\cup Q)$, $(X',Y'\cup\{y,z\})\in\SE_A^B(P)
\setminus \SE_A^B(Q)$.
\end{enumerate}
}
\begin{proof}
Since $(X,Y)\in \SE_A^B(P)$, $X\subseteq Y\subseteq \At(P)\cup A$ (cf.
Lemma \ref{stin0}). Thus, $X\subseteq Y\subseteq \At(P\cup Q)\cup A$.

By applying repeatedly Lemma \ref{stin1a} and then Lemma \ref{stin1b},
we can construct sets $X'\subseteq\at(P\cup Q)$ and $Y''\subseteq 
A\cup\at(P\cup Q)$ 
such 
that\\
\hspace*{0.35in}\parbox{5in}{
\begin{enumerate}
\item[a.] $(X',Y'')\in\SE_A^B(P)\setminus \SE_A^B(Q)$, and
\item[b.] $|Y''\setminus\at(P\cup Q)|\leq 2$.
\end{enumerate}
}
If $Y''\subseteq\at(P\cup Q)$, (i) follows (with $Y'=Y''$). 
Otherwise, (ii)
follows 
(by Lemma \ref{stin1c}).
\end{proof}

\smallskip
Next we present a proof of Lemma \ref{lemma:ue}

\medskip
\noindent
\textit{Lemma \ref{lemma:ue}}\\
\textit{Let $P$ be a program and $A,B\subseteq \At$. If $\at(P)\subseteq A$
and $\at(P)\cap B=\emptyset$, then $(X,Y)\in\SE_A^{B}(P)$ if and only
if there are $X',Y'\subseteq\at(P)$ and $W\subseteq A\setminus \at(P)$
such that one of the following conditions holds:
\begin{enumerate}
\item[i.] $X=X'\cup W$, $Y=Y'\cup W$, and $(X',Y')\in\SE_A^{B}(P)$
\item[ii.] $X=X'\cup W$, $(X',X')\in\SE_A^{B}(P)$ and $Y=X'\cup
W\cup \{y\}$, for some $y\in A\setminus\at(P)$
\item[iii.] $X=X'\cup W$, $(X',X')\in\SE_A^{B}(P)$ and $Y=X'\cup
W\cup D$, for some $D\subseteq B\cap(A\setminus\at(P))$ such that
$W\cap D=\emptyset$ and $|D| \geq 2$.
\end{enumerate}
}
\begin{proof}
($\Leftarrow$) If (i) holds, then $(X,Y)\in\SE_A^{B}(P)$ follows from
Lemma \ref{stin1a}. Thus, let us assume that (ii) or (iii) holds. Then 
$X'\models P$ and so, $X'\cup\{y\}\cup W\models P$ (respectively,
$X'\cup W\cup D\models P$). Moreover, $X\subset Y$. Thus,
since $Y\subseteq A$, the condition (2) for $(X,Y)\in\SE_A^{B}(P)$ 
holds. Next, it is evident that the condition (3) is vacuously true. 
The condition (4) is also vacuously true. To see it,
let us consider $Z\subset Y$ such that $Z|_B\subset X|_B$ and 
$Z|_A\supseteq X|_A$, or $Z|_B\subseteq X|_B$ and $Z|_A\supset X|_A$.
Since $X\subseteq Y\subseteq A$, $X\subseteq Z$. Thus, $X|_B\subseteq Z|_B$, and
so $Z|_B\subset X|_B$ is impossible. Consequently, $Z|_B\subseteq
X|_B$ and $Z|_A\supset X|_A$. The latter implies $X\subset Z$. We also
have $Z\subset Y$. Thus, $|Y\setminus X|\geq 2$, contradicting (ii).
It follows that (iii) holds. Consequently, $X'\cup W\subset Z\subset
X'\cup W\cup D$. Since $Z|_B\subseteq X|_B$, $D=\emptyset$, a 
contradiction. 

Finally, let $Z$ be a set verifying the condition (5) for $(X',X')\in
\SE_A^{B}(P)$ (which holds under either (ii) or (iii)). Clearly, the set 
$Z\cup W$ demonstrates that the condition (5) for 
$(X,Y)\in\SE_A^{B}(P)$ holds.

\smallskip
\noindent
($\Rightarrow$) Let $W=X\cap (A\setminus \at(P))$. We define $X'=X
\setminus W$ and $Y'=Y\setminus W$. Clearly, $X'\subseteq\at(P)$.
Moreover, by Lemma \ref{stin1a}, $(X',Y')\in\SE_A^{B}(P)$.
If $Y'\subseteq\at(P)$, then (i) follows. 

Thus, let us assume that $Y'\setminus\at(P)\not=\emptyset$. Next, let 
us assume that $X'\subset Y'\cap \at(P)$ and let $Z=Y'\cap \at(P)$.
Clearly, $Z\subset Y'$, $Z|_B=\emptyset$ and $X|_A=X\subset Z=Z|_A$. 
By the condition (4) for $(X',Y')\in\SE_A^{B}(P)$, $Z\not\models P^{Y'}$.
On the other hand, by the condition (1) for $(X',Y')\in\SE_A^{B}(P)$, 
$Y'\models P$. Consequently, $Y'\models P^{Y'}$. It follows that $Z
\models P^{Y'}$, a contradiction.

It follows that $X'=Y'\cap\at(P)$. If
there are $y',y''\in Y'\setminus\at(P)$ such that $y'\not=y''$ and 
$y'\notin B$, then let us define $Z=X'\cup\{y'\}$. It is easy to verify
that $Z$ contradicts the condition (4). If $|Y'\setminus\at(P)|=1$, then
(ii) follows (with the only element of $Y'\setminus\at(P)$ as $y$).
Otherwise, $|Y'\setminus\at(P)|\geq 2$ and $Y'\setminus\at(P) \subseteq B$.
In this case, (iii) follows (with $D=Y'\setminus\at(P)$). 
\end{proof}

\label{lastpage}
\end{document}